\documentclass{article} 
\usepackage{iclr2025_conference,times}
\iclrfinalcopy

\usepackage{bm}
\usepackage{bbm}
\usepackage[mathscr]{eucal}

\newcommand{\ie}{\textit{i}.\textit{e}.,}
\usepackage{bm}
\usepackage{bbm}
\usepackage{amsthm}
\usepackage[mathscr]{eucal}

\newcommand{\relu}{\text{ReLU}}

\newtheorem*{remark}{Remark}

\newcounter{subassumption}[asu]

\def\eq{\textbf{=}}

\newtheorem{theorem}{Theorem}
\newtheorem{lemma}{Lemma}
\newtheorem{definition}{Definition}
\newtheorem{proposition}{Proposition}

\newcommand{\norm}[1]{\left\|#1\right\|}
\newcommand{\abs}[1]{\left|#1\right|}
\newcommand{\R}{\mathbb{R}}
\newcommand{\pr}{\mathbb{P}}

\newcommand{\iden}{\mathbb{I}}

\def\1{\bm{1}}

\def\vZero{{\bm{0}}}
\def\vOne{{\bm{1}}}

\def\va{{\bm{a}}}
\def\vb{{\bm{b}}}

\def\vf{{\bm{f}}}
\def\vg{{\bm{g}}}

\def\vp{{\bm{p}}}

\def\vs{{\bm{s}}}

\def\vv{{\bm{v}}}

\def\vx{{\bm{x}}}
\def\vy{{\bm{y}}}
\def\vz{{\bm{z}}}

\def\mW{{\bm{W}}}
\def\mX{{\bm{X}}}

\def\mZero{{\bm{0}}}

\DeclareMathAlphabet{\mathsfit}{\encodingdefault}{\sfdefault}{m}{sl}
\SetMathAlphabet{\mathsfit}{bold}{\encodingdefault}{\sfdefault}{bx}{n}

\def\cD{{\mathcal{D}}}

\def\cO{{\mathcal{O}}}

\def\cV{{\mathcal{V}}}

\def\bOne{{\mathbbm{1}}}

\usepackage{hyperref}
\usepackage{url}

\usepackage{booktabs}       
\usepackage{amsfonts}       
\usepackage{nicefrac}       
\usepackage{microtype}      
\usepackage{xcolor}         
\usepackage{amsmath} 
\usepackage{svg}

\usepackage{thmtools, thm-restate}
\usepackage{tikz}
\usepackage{subcaption}
\usepackage{wrapfig}
\usepackage[shortlabels]{enumitem}
\usepackage{graphicx}
\usepackage{rotating}

\usepackage{cleveref}
\crefname{enumi}{example}{examples}

\usepackage[most]{tcolorbox}
\newtcolorbox{highlight}[1][]{
    enhanced,
    colback=yellow!10,
    colframe=gray!30,
    boxrule=0.5pt,
    arc=2pt,
    leftrule=3pt,
    rightrule=3pt,
    toprule=1pt,
    bottomrule=1pt,
    breakable,
    #1
}

\makeatletter
\def\blfootnote{\xdef\@thefnmark{}\@footnotetext}
\makeatother

\title{
    Everything Everywhere All at Once:\\ LLMs can In-Context Learn\\ Multiple Tasks in Superposition}

\author{Zheyang Xiong$^w$, Ziyang Cai$^w$, John Cooper$^w$, Albert Ge$^w$, Vasilis Papageorgiou$^w$ \\ \textbf{Zack Sifakis}$^w$, \textbf{Angeliki Giannou}$^w$, \textbf{Ziqian Lin}$^w$, \textbf{Liu Yang}$^w$, \textbf{Saurabh Agarwal}$^w$ \\ \textbf{Grigorios G Chrysos}$^w$, \textbf{Samet Oymak}$^m$, \textbf{Kangwook Lee}$^w$, \textbf{Dimitris Papailiopoulos}$^{w,ms}$ \\ \\
\textnormal{$^{w}$University of Wisconsin-Madison},
\textnormal{$^{m}$University of Michigan},
\textnormal{$^{ms}$Microsoft Research}
}

\begin{document}

\maketitle

\begin{abstract}
Large Language Models (LLMs) have demonstrated remarkable in-context learning (ICL) capabilities. In this study, we explore a surprising phenomenon related to ICL: LLMs can perform multiple, computationally distinct ICL tasks simultaneously, during a single inference call, a capability we term ``task superposition''. We provide empirical evidence of this phenomenon across various LLM families and scales and show that this phenomenon emerges even if we train the model to in-context learn one task at a time. We offer theoretical explanations that this capability is well within the expressive power of transformers. We also explore how LLMs internally compose task vectors during superposition. Furthermore, we show that larger models can solve more ICL tasks in parallel, and better calibrate their output distribution. Our findings offer insights into the latent capabilities of LLMs, further substantiate the perspective of ``LLMs as superposition of simulators'', and raise questions about the mechanisms enabling simultaneous task execution.
\end{abstract}

\blfootnote{Email: \texttt{<zheyang@cs.wisc.edu>}. Correspondence: \texttt{<dimitris@papail.io>}.}
\blfootnote{Code available at: \texttt{github.com/edixiong/task-superposition}}

\section{Introduction} \label{intro}

Large Language Models (LLMs) have demonstrated remarkable capabilities across various domains, with one of the most intriguing being in-context learning (ICL). ICL enables LLMs to perform  tasks during inference without the need to fine-tune for that particular task, simply by providing a few examples within the input prompt. This ability has sparked significant interest in the research community, as it suggests that LLMs can adapt to novel tasks on-the-fly, using the capabilities that they acquired during pretraining, and the context provided.

While ICL has been extensively studied from both theoretical and empirical perspectives,
many aspects of its underlying mechanisms remain elusive. In this work, we study a surprising phenomenon related to ICL that, to the best of our knowledge, has not been thoroughly studied before: LLMs can perform multiple distinct ICL tasks simultaneously, in a single inference call, a capability we refer to as {\it ``task superposition''}, with examples shown in Figure~\ref{fig:fulldemo}.

Our study suggests that pretrained autoregressive LLMs such as Llama \citep{touvron2023llama} or GPT-3.5\footnote{In particular, \texttt{gpt-3.5-turbo-instruct}.} \citep{brown_language_2020} display superposition of tasks purely \emph{in-context}. When presented with multiple in-context examples from different tasks, in the same prompt, the models can generate outputs that correspond to solutions for all these individual tasks. For instance, given examples of addition and translation, the model can concurrently produce correct answers for both tasks, as well as the composition of these tasks (e.g., the result of addition translated into another language).

\begin{figure}[t!]
    \centering
    \begin{subfigure}[t]{1\textwidth}
        \centering
        \includegraphics[width=1\linewidth]{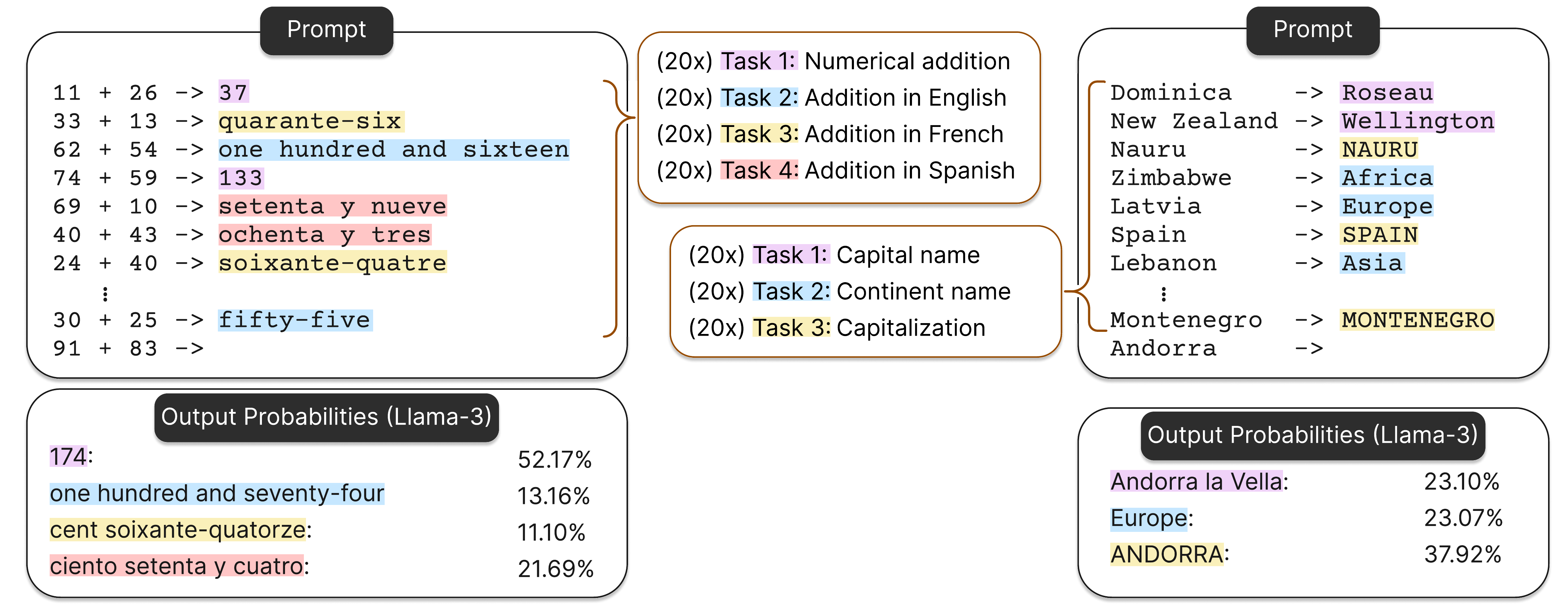}
        \caption{\textbf{(left)} Two-digit addition in a variety of languages. \textbf{(right)} Naming the capital of a given country name, naming the continent of a given country name or capitalizing the country name.}
        \label{fig:llama3_demo}
    \end{subfigure}
    \begin{subfigure}[t]{1\textwidth}
        \centering
        \includegraphics[width=1\linewidth]{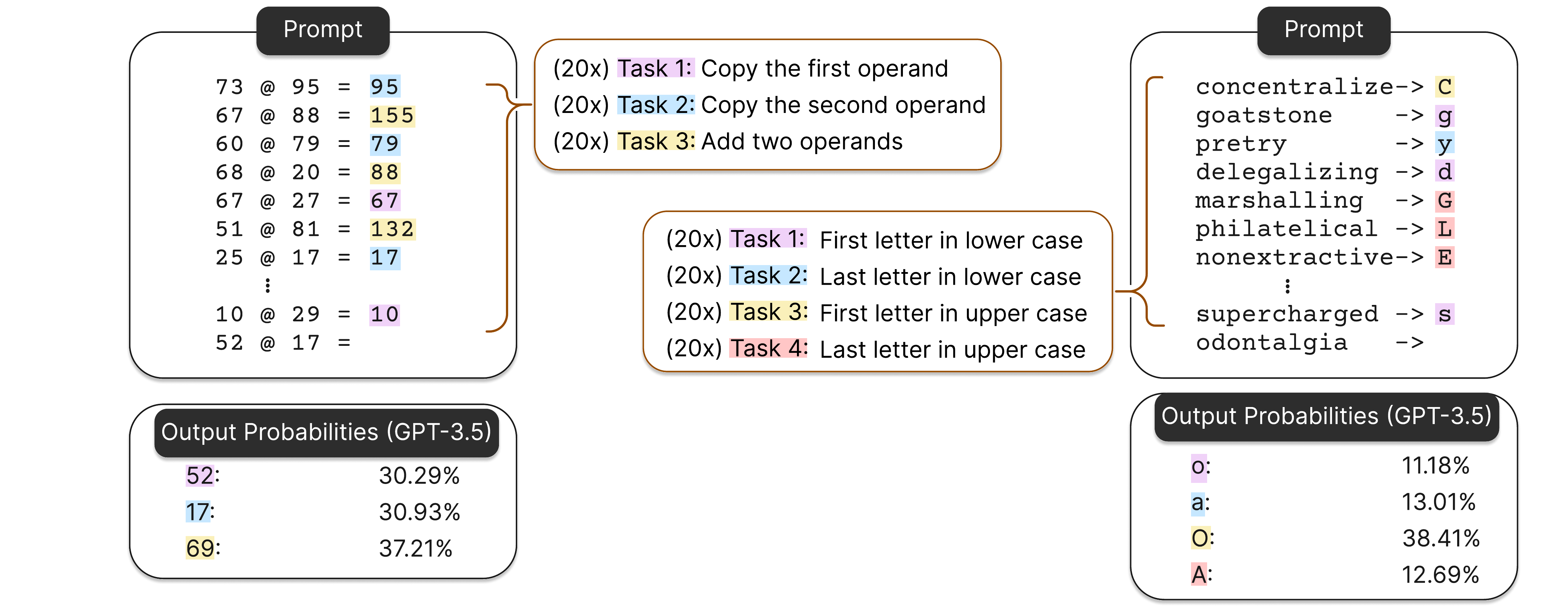}
        \caption{\textbf{(left)} Tasks \texttt{copy(op1)}, \texttt{copy(op2)} and \texttt{op1+op2}. \textbf{(right)} First or last letter in upper or lower case.}
        \label{fig:gpt3_demo}
    \end{subfigure}
    \caption{LLMs can perform task superposition. \textbf{(a)} Llama-3 70B and \textbf{(b)} GPT-3.5 Turbo are each presented with two sets of tasks. For each set of tasks, we show an example prompt such that all except the last row are in-context examples of one of the tasks and the last row is the query. We provide $20$ in-context task examples for each task in the prompt with order randomized and provide the probabilities of outputs when correctly performing each task on the query.
    }
    \label{fig:fulldemo}
\end{figure}

Figure \ref{fig:fulldemo} illustrates this phenomenon. In Figure \ref{fig:llama3_demo} (left), given in-context examples of addition in different languages and the query ``91 + 83 $\rightarrow$'', the model generates probabilities for the correct sum in various languages, demonstrating its ability to perform addition and translation concurrently.

This discovery aligns and  lends further support to  the view of LLMs as superposition of simulators \citep{janus,shanahan2023role,nardo2023the} and the Bayesian perspective of ICL proposed by \citet{Xie_Raghunathan_Liang_Ma_2022}. While not a mathematically rigorous formulation, we can conceptualize the output of an LLM as a weighted sum of conditional probabilities across possible tasks: 
\begin{equation*}
\pr(\mathsf{output}|\mathsf{prompt}) \approx \sum_{\mathsf{task}}    \pr(\mathsf{output}|\mathsf{task},\mathsf{prompt})\pr(\mathsf{task}|\mathsf{prompt}).
\end{equation*}
In this conceptual model, $\pr(\mathsf{output}|\mathsf{prompt})$ represents the probability distribution over possible outputs given the input prompt,
a $\mathsf{task}$ can be thought of as a latent variable representing different capabilities the model might possess (e.g., arithmetic, translation, sentiment analysis), $\pr(\mathsf{output}|\mathsf{task},\mathsf{prompt})$ represents the output probability distribution if the model was specifically attempting to solve a single task, based on the test example in the prompt, and $\pr(\mathsf{task}|\mathsf{prompt})$ represents the model's inferred probability that the prompt specifies a particular task. 

Although this mental model is an over-simplification of how an LLM operates, it offers a clean conceptual framework for the task superposition phenomenon we observe. Our findings lend support to the idea that LLMs can simultaneously maintain and utilize multiple task distributions, resulting in outputs that reflect a combination of relevant tasks.

\paragraph{Our Contributions:} 
Our study makes several key contributions:
\begin{enumerate}[topsep=-1pt,itemsep=-0.1ex,partopsep=1ex,parsep=1ex,leftmargin = 4ex]
\item Through extensive empirical investigation and theoretical results, we demonstrate that task superposition is prevalent across various pretrained LLM families (GPT3.5, LLama-3, Qwen).
\item We empirically show that task superposition can emerge, even if we train on one task at a time.
\item We provide a theoretical construction showing that Transformers models are indeed capable of task superposition, and have the capacity to implement multiple tasks in parallel.
\item We explore how LLMs internally compose task vectors \citep{hendel_-context_2023} during superposition, and show how convex combinations of task vectors can reproduce the superposition effect.
\item We show that larger models can solve more tasks in parallel and more accurately reflect the distribution of in-context tasks.
\end{enumerate}

We believe that our findings offer new insights into the latent capabilities of LLMs and raise  questions about the mechanisms enabling simultaneous task execution. We believe this work sheds more light on the ICL capabilities of frontier language models, and offers a glimpse on potential applications of task superposition in practical settings.

\section{Related Work} \label{relatedwork}

\textbf{Theory and practice of in-context learning.} There is rich literature which formalizes in-context learning under diverse definitions. For example, prior works study in-context learning through a Bayesian framework for task retrieval \citep{Xie_Raghunathan_Liang_Ma_2022, Panwar_Ahuja_Goyal_2023,Zhang_Zhang_Yang_Wang_2023}, martingales \citep{Falck_Wang_Holmes_2024}, optimizers \citep{akyurek_what_2023,oswald_transformers_2023,Dai_Sun_Dong_Hao_Sui_Wei_2022} and more \citep{Reddy_2024,Olsson_Elhage_Nanda_Joseph_DasSarma_Henighan_Mann_Askell_Bai_Chen_etal._2022}. Other works confirm the theoretical framing of in-context learning by using it to implement a variety of algorithms and methods \citep{Zhou_Bradley_Littwin_Razin_Saremi_Susskind_Bengio_Nakkiran_2023,Ahn_Cheng_Daneshmand_Sra_2023, Giannou_Rajput_Sohn_Lee_Lee_Papailiopoulos_2023, Wu_Zou_Chen_Braverman_Gu_Bartlett_2024, Laskin_Wang_Oh_Parisotto_Spencer_Steigerwald_Strouse_Hansen_Filos_Brooks_etal._2022,Zhou_Nova_Larochelle_Courville_Neyshabur_Sedghi_2022}, or to approximate general-purpose computing machines \citep{Giannou_Rajput_Sohn_Lee_Lee_Papailiopoulos_2023, Wei_Chen_Ma_2022}.

To bridge the gap between theory and practice, many works have used these theoretical insights to study in-context learning behaviors, such as in many-shot in-context learning, \citep{Agarwal_Singh_Zhang_Bohnet_Chan_Anand_Abbas_Nova_Co-Reyes_Chu_etal._2024}, long-context \citep{Li_Zhang_Do_Yue_Chen_2024}, or eliciting personas \citep{Choi_Li_2024}. Other works study the factors that influence how well models can learn through context, such as task diversity \citep{raventos2023effects,Chan_Santoro_Lampinen_Wang_Singh_Richemond_McClelland_Hill_2022}, the balance between pre-training priors and in-context \citep{Wei_Wei_Tay_Tran_Webson_Lu_Chen_Liu_Huang_Zhou_etal._2023,lin_dual_2024}, in-context labels \citep{Min_Lyu_Holtzman_Artetxe_Lewis_Hajishirzi_Zettlemoyer_2022,Lyu_Min_Beltagy_Zettlemoyer_Hajishirzi_2022}, and the in-context format \citep{Lampinen_Dasgupta_Chan_Matthewson_Tessler_Creswell_McClelland_Wang_Hill_2022}. In-context learning has also been proposed as a means of fine-tuning to improve non-language tasks \citep{Dinh_Zeng_Zhang_Lin_Rajput_Gira_Sohn_Papailiopoulos_Lee_2022}.

The development of new architectures such as state space models \citep{Gu_Dao_2023} has further motivated studying whether in-context learning is prevalent in alternative architectures such as Mamba \citep{Park_Park_Xiong_Lee_Cho_Oymak_Lee_Papailiopoulos_2024,Grazzi_Siems_Schrodi_Brox_Hutter_2024,Zeng_Kang_Chen_Koo_Lee_2024} or in looped transformers \citep{Yang_Lee_Nowak_Papailiopoulos_2023}.

Steering models through in-context learning has been a growing area of interest.
Recent work has hypothesized that in-context learning can be encapsulated by a high-dimensional description of a task, which can be used to replace, \citep{hendel_-context_2023} compose \citep{Todd_Li_Sharma_Mueller_Wallace_Bau_2024} or augment \citep{Liu_Ye_Xing_Zou_2024} the latent states of a model, in order to alter its default behavior. Task vectors can be combined via arithmetic operations to solve a variety of tasks \citep{ilharco_editing_2023}. Prior work has also been investigating the power of tokens in defining a task \citep{Bai_Huang_Piano_Rondeau_Chen_Gao_Cheung_2024}.

\textbf{Other definitions of superposition.}
Our findings on superposition are inspired by notions of language models as multiverse generators \citep{reynolds_multiversal_2021,moire_2021}. One consequence of LLMs as a superposition of tasks is that the outputs may collapse to unintended simulacra, a behavior known as the ``Waluigi effect'' \citep{nardo2023the}.

Superposition has been defined in various related contexts of learning models. Feature superposition \citep{elhage_toy_2022} refers to a neural network's ability to represent multiple learned concept in a single neuron. Though our discovery of task superposition describes the same abstract idea, we stress that it is distinct from feature superposition because task superposition is most apparent in the final output of a model. Feature superposition is a microscopic-level observation whereas task superposition is a macroscopic-level observation. 

Superposition is also described as a way to store multiple models in a single set of parameters \citep{cheung_superposition_2019}, processing multiple inputs simultaneously \citep{Shen_Fan_Pratt_Park_Wallingford_Kakade_Holtzman_Krishna_Farhadi_Kusupati_2024,Murahari_Jimenez_Yang_Narasimhan_2022}. In our work, we demonstrate task superposition directly as a result of language pre-training, without the necessity of additional adapters or decoding strategies. 

\section{LLMs are a superposition of multiple in-context learners} \label{sec:llmssuperposition}

\begin{figure}[t!]
    \centering
    \begin{subfigure}[t]{1\textwidth}
        \centering
        \includegraphics[width=0.9\linewidth]{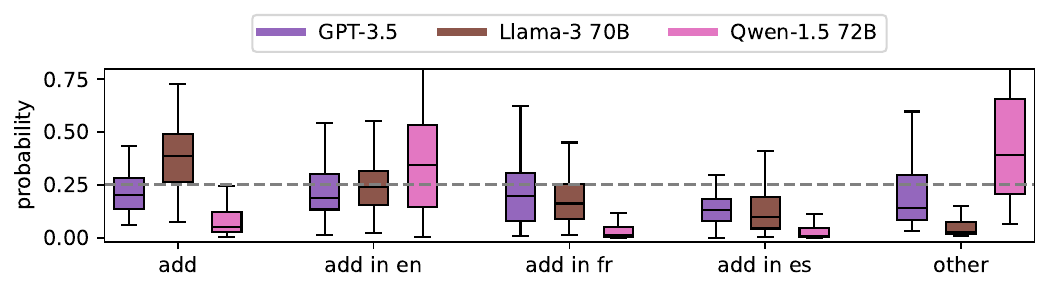}
        \caption{Addition in original numerical form and in different languages as in Figure \ref{fig:llama3_demo} (left).}
        \label{fig:add_translate}
    \end{subfigure}
    \begin{subfigure}[t]{1\textwidth}
        \centering
        \includegraphics[width=0.9\linewidth]{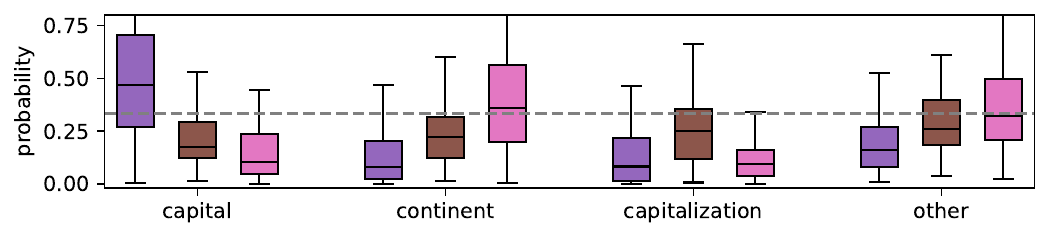}
        \caption{Capital name, continent name and capitalization as in Figure \ref{fig:llama3_demo} (right).}
        \label{fig:country}
    \end{subfigure}
    \begin{subfigure}[t]{1\textwidth}
        \centering
        \includegraphics[width=0.9\linewidth]{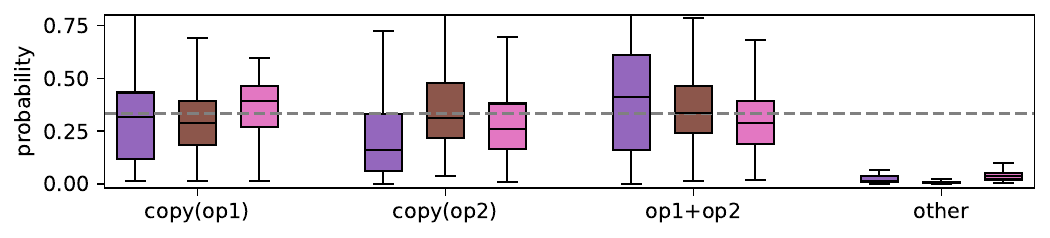}
        \caption{\texttt{copy(op1), \texttt{copy(op2)}}, and \texttt{op1+op2} as in Figure \ref{fig:gpt3_demo} (left).}
        \label{fig:AB}
    \end{subfigure}
    \begin{subfigure}[t]{1\textwidth}
        \centering
        \includegraphics[width=0.9\linewidth]{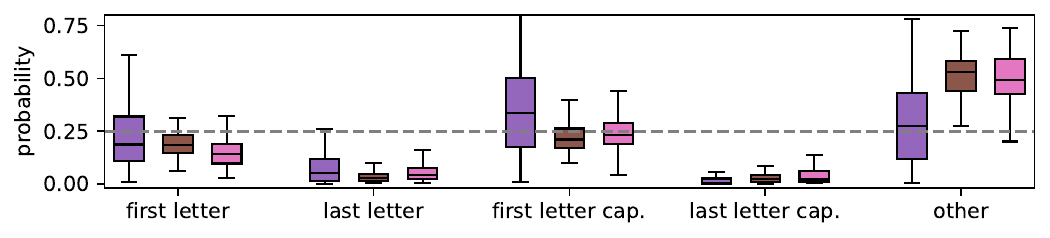}
        \caption{First or last letter in upper or lower cases as in Figure \ref{fig:gpt3_demo} (right).}
        \label{fig:letter}
    \end{subfigure}
    \caption{Distributions of probabilities for correct outputs for each task in box plots, where 0/25/50/75/100-percentiles are shown and region between 25 and 75 percentiles is colored. For every set of tasks, we tested with $100$ prompts and for each prompt, every task has $20$ random in-context task examples with order randomized like in Figure \ref{fig:fulldemo}. Category \texttt{other} is the sum of probabilities of all other outputs. Gray dashed line in each figure is the ideal probability if we assume the model perfectly calibrates its output distribution to the distribution of in-context task examples. With uniform distribution of task examples, the dashed lines are at $0.25$ (4 tasks setting) and $0.33$ (3 tasks setting).}
    \label{fig:op1op2_out}
    \vspace{-10pt}
\end{figure}

In this section, we want to investigate if existing pre-trained models exhibit superposition of multiple tasks and whether this phenomenon is common (i.e., whether we can observe this phenomenon on a variety set of tasks and different families of LLMs). 

\begin{highlight}
    \paragraph{Finding 1:} 
    \emph{LLMs can in-context learn multiple tasks in superposition when provided with prompts of a mixture of task examples.}
\end{highlight}

We denote $K$ by the number of tasks and consider four different settings of task mixtures.
\begin{enumerate}
    \item Numerical addition and addition in English, French or Spanish $(K=4)$. Example prompt is shown in Figure \ref{fig:llama3_demo} (left).
    \item Given a name of a country, name the capital, continent or capitalize the country name $(K=3)$. Example prompt is shown in Figure~\ref{fig:llama3_demo} (right).
    \item Given input ``\texttt{\{op1\}@\{op2\}}'', copy \texttt{op1}, \texttt{op2} or add \texttt{op1} and \texttt{op2} $(K=3)$. Example prompt is shown in Figure \ref{fig:gpt3_demo} (left).
    \item Given a word, output first letter or last letter in lower or upper cases $(K=4)$. Example prompt is shown in Figure~\ref{fig:gpt3_demo} (right).
\end{enumerate}

We provide GPT-3.5 \citep{brown_language_2020}, Llama-3 70B \citep{llama3modelcard} and Qwen-1.5 72B \citep{qwen} with prompts of uniform mixture of tasks (each task has $20$ random examples in the prompt ordered randomly). For each prompt consisting of in-context task examples (e.g., ``$11+26\rightarrow 37$'' for the first task in the first setting) and a query (e.g., ``$91+83\rightarrow$''), we calculate the probabilities of outputs when correctly performing each task on the query and plot the distribution of probabilities for each task in Figure \ref{fig:op1op2_out}. Details on calculating the probabilities is in Appendix~\ref{app:prob_details}.

Figure \ref{fig:op1op2_out} reveals that in all four sets of tasks, all models have non-negligible median values of probabilities for at least two tasks. This indicates that the models can in-context learn multiple tasks in superposition when provided with prompts of a mixture of task examples.

We can also observe that, even though every task in a prompt has an equal number of in-context examples ($20$ examples), LLMs do not calibrate their output distribution perfectly with the in-context task example distribution and they still have bias on what task to perform. For example, Figure \ref{fig:add_translate} shows that Llama-3 70B prefers performing numerical addition over addition in other languages, Qwen-1.5 72B prefers addition in English while GPT-3.5 does not have a strong preference over a single task. On the other hand, in Figure \ref{fig:country} GPT-3.5 has a strong preference over the \texttt{capital} task.

Additionally, some tasks are ``harder'' than other tasks. For example, in Figure \ref{fig:letter}, all models assign near-zero probability for task answers of \texttt{last\_letter} and \texttt{last\_letter\_cap}. The category \texttt{other} has relatively high values, indicating a high noise when prompted with in-context examples of this setting. In contrast, in Figure \ref{fig:AB}, category \texttt{other} has very small values, indicating that all models most of the time would correctly assign the output probabilities to the correct answers.

\section{Task superposition in models trained from scratch}
\label{sec:fromscratch}

In Section \ref{sec:llmssuperposition} we investigated task superposition in pre-trained LLMs at inference time. In this section, we further investigate how task superposition emerges in LLMs during training. Specifically, if we train the model to in-context learn one task at a time, can it perform task superposition when provided with prompts containing examples of multiple tasks?

To answer this question, we train a small GPT-2 model (12 heads, 12 layers and $\sim$86 million parameters) \citep{radford2019language} to learn a family of retrieval tasks. The input has the form ``\texttt{\{ch1\}\{ch2\}\{ch3\}\{ch4\}\{ch5\}\{ch6\}\{ch7\}\{ch8\}$\rightarrow$}'' where \texttt{ch1}, ..., \texttt{ch8} are distinct single characters. We consider $8$ retrieval tasks -- \texttt{ret1}, ..., \texttt{ret8} -- where \texttt{ret1} is to output \texttt{ch1} and so on. The model is trained to in-context learn one task (retrieve one of $\{\texttt{ch1},...,\texttt{ch8}\}$) at a time in training. Namely, during training, the model is only provided with text data such that each prompt only contains in-context examples of a single randomly chosen task (and different prompts can correspond to different tasks).

Concretely, for each sample, we randomly select task $t\in \{\texttt{ret1},...,\texttt{ret8}\}$ and inputs $\vx^{(1)},...,\vx^{(m)}$, where each $\vx^{(j)}$ is an eight-character long string. We then form the sequence $\vs = [\vx^{(1)},\vg_t(\vx^{(1)}), ..., \vx^{(m)}, \vg_t(\vx^{(m)})]$ where $\vg_t(\vx^{(j)})$ is the output of performing task $t$ on $\vx^{(j)}$. We train the model $M_\theta$ parametrized by $\theta$ using ICL training. In particular, we minimize the following objective:

\begin{equation}
    \label{eq:icl_training}
    \min_{\theta}\mathbf{E}_\vs\left(\frac{1}{m-1}\sum_{j=1}^{m-1}\text{CE}(M_\theta(\vs_j\oplus\vx^{(j+1)}), \vg_t(\vx^{(j+1)}))\right),
\end{equation}
where $\vs_j\oplus\vx^{(j+1)}\equiv[\vx^{(1)},\vg_t(\vx^{(1)}), ..., \vx^{(j)},\vg_t(\vx^{(j)}), \vx^{(j+1)}]$ and $\text{CE}$ is the cross-entropy loss.

\begin{figure}[t!]
    \centering
    \begin{subfigure}[t]{0.46\textwidth}
        \centering
        \includegraphics[width=1\linewidth]{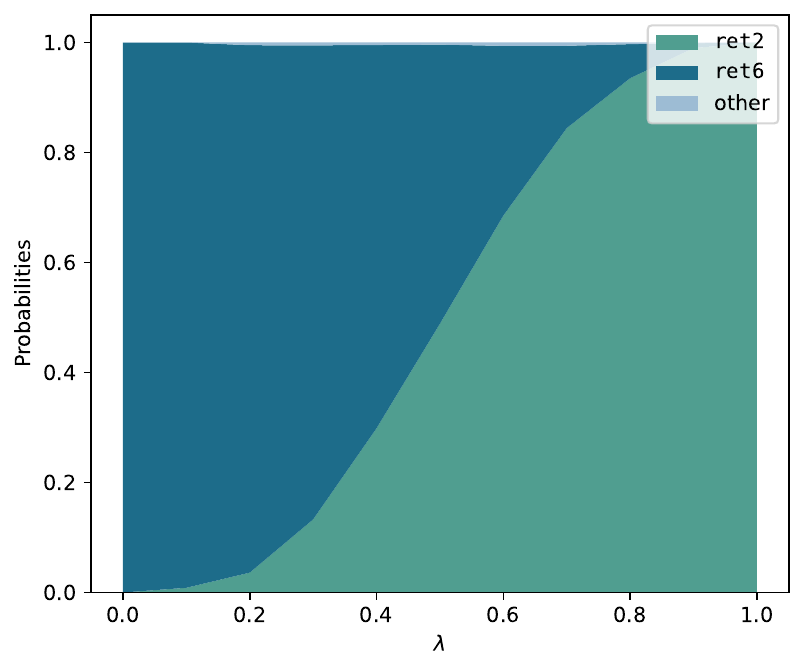}
        \captionsetup{width=0.9\textwidth}
        \vspace{-5pt}
        \caption{\centering Trained on retrieval tasks and tested on prompts with mixture of in-context examples of \texttt{ret2} and \texttt{ret6}.} 
        \label{fig:scratch_ret}
    \end{subfigure}
    \begin{subfigure}[t]{0.46\textwidth}
        \centering
        \includegraphics[width=1\linewidth]{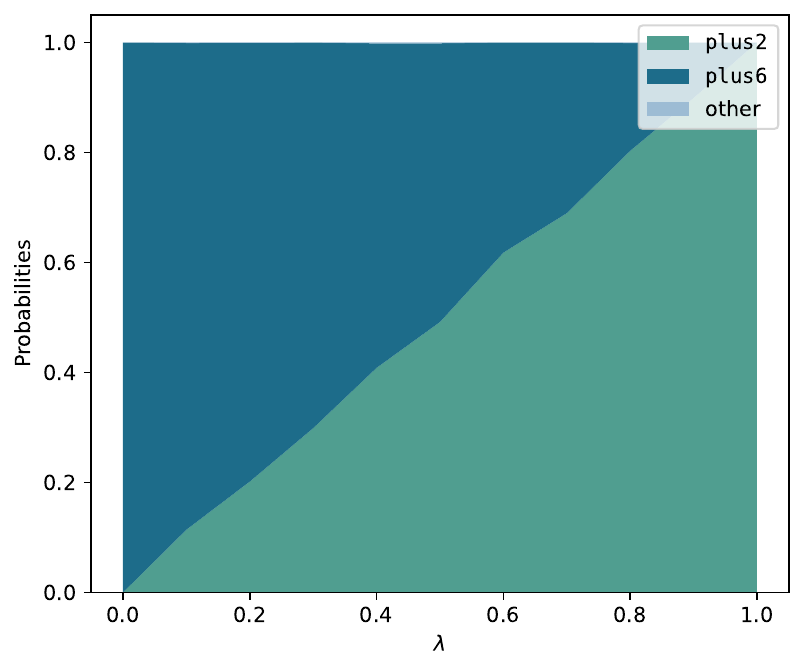}
        \captionsetup{width=0.9\textwidth}
        \vspace{-5pt}
        \caption{\centering Trained on addition tasks and tested on prompts with mixture of in-context examples of \texttt{plus2} and \texttt{plus6}.}
        \label{fig:scratch_plus}
    \end{subfigure}
    \vspace{-5pt}
    \caption{We consider two different training settings of tasks: (a) given an eight-character length string as input, consider \texttt{ret1}, ..., \texttt{ret8} where \texttt{ret1} is to retrieve the first character and so on; and (b) given a two-digit integer as an input, consider \texttt{plus0}, ..., \texttt{plus9} where \texttt{plus0} is to add $0$ on the input and so on. The model in-context learn one task at a time during training. After training, for each setting, we select two tasks ad we provide the model with prompts containing in-context examples from these two tasks and vary the mixture ratio $\lambda$ such that the in-context task example distribution for two tasks is $[\lambda, 1-\lambda]$. We plot $\lambda$ on x-axis and the output probabilities of task answers for each task on y-axis.}
    \label{fig:main}
    \vspace{-10pt}
\end{figure}

After training, we provide the model with prompts containing in-context examples of two tasks (in particular, we choose \texttt{ret2} and \texttt{ret6}) and see if the model performs task superposition. We vary the proportion of in-context examples of two tasks and plot the output distributions in Figure \ref{fig:scratch_ret}.

Similarly, we consider a second setting involving $10$ tasks. Given a two digit integer input \texttt{num}, task \texttt{plus0} outputs \texttt{num}, task \texttt{plus1} outputs $\texttt{num}+1$ and so on, up to task \texttt{plus9}. The model is trained to in-context learn one of  \texttt{plus0},..., \texttt{plus9} at a time, following the procedure above. During inference time, the model is tested with prompts containing a mixture of in-context examples from tasks \texttt{plus2} and \texttt{plus6}. We vary the mixture ratio and show the output distributions in Figure \ref{fig:scratch_plus}.

\begin{highlight}
    \paragraph{Finding 2:} 
    \emph{Transformers can in-context learn multiple tasks in superposition even if trained to in-context learn one task at a time.}
\end{highlight}

Remarkably, from Figure \ref{fig:scratch_ret} and \ref{fig:scratch_plus}, GPT-2 trained from scratch that in-context learns one task at a time can generalize to simultaneously performing multiple tasks and calibrate the output probabilities according to the in-context task example distribution when provided with a mixture of in-context examples. For example, in Figure \ref{fig:scratch_ret} at the mixture ratio $\lambda=0.5$, meaning that $50$ percent of the examples in the prompt is from task \texttt{ret2} and the other $50$ percent comes from task \texttt{ret6}, we can see the output probabilities for task answers of \texttt{ret2} and \texttt{ret6} being roughly $[0.5, 0.5]$. We can also observe similar behavior in Figure \ref{fig:scratch_plus}.

\section{Transformers have the capacity to perform task superposition}
\label{sec:construction}
In this section, we explore whether Transformers have the inherent expressivity to perform multiple tasks in superposition with a single inference call. To this end, we provide a theoretical construction of a Transformer which, given the ability to implement multiple tasks, performs task superposition depending on the examples given in-context.

\begin{restatable}{theorem}{constructtheorem}
\label{thm:constructtheorem}
    A seven layer transformer with embedding dimension $\cO(d + \log(mn))$ with $K$ heads per attention layer can perform $K$ tasks on vectors of dimension $d$ in superposition, with weighting based on $m$ different in-context examples each of length $n$ .
\end{restatable}

The proof of Theorem \ref{thm:constructtheorem} is provided in Appendix \ref{subsec:superposed_tasks}. Note that while this does not guarantee that training a Transformer will actually find these parameters, it does indicate that Transformers are expressive enough to perform task superposition at test time. Below we outline the main ideas used in the proof.
  \paragraph{Prediction based on multiple tasks.} Assume that we are given $m$ in-context samples $(\vx_1^{(j)},\hdots,\vx_{n-2}^{(j)}, \text{`}\eq\text{'}, \vy^{(j)})_{j=1}^m$  where `$\eq$' represents a specific value used only for preceding the label, and a set of $k$ different Transformers $\text{TF}_i$ which can implement the $T$ different desired tasks, where each deterministic task is denoted as $\vg_i(\vx^{(j)})$ with $i\in[k]$ and $j\in[m]$, \textit{i.e.} $\vy^{(j)}=\vg_i(\vx^{(j)})$ for some task $i$ dependent on sample $j$. Using the weights of each $\text{TF}_i$, we can compute outputs of the following form:
\begin{equation*}
\resizebox{1\linewidth}{!}{$
    \begin{bmatrix}
        \hdots & \vx_1^{(j)}  & \hdots & \vx_{n-2}^{(j)} & \eq & \vy^{(j)} & \hdots\\
        \hdots  & \vZero & \hdots & \vZero & \vZero & \vZero & \hdots \\
          & \vdots &  & \vdots & \vdots & \vdots & \\
        \hdots & \vZero & \hdots & \vZero & \vZero & \vZero & \hdots
\end{bmatrix} \to \begin{bmatrix}
        \hdots & \vx_1^{(j)} & \hdots & \vx_{n-2}^{(j)} & \eq & \vy^{(j)} & \hdots\\
        \hdots & \vZero & \hdots & \vZero & \vZero &\|\vg_1(\vx^{(j)}) - \vy^{(j)}\|_1& \hdots \\
        & \vdots & & \vdots & \vdots & \vdots & \\
        \hdots & \vZero & \hdots & \vZero  & \vZero &\|\vg_T(\vx^{(j)}) - \vy^{(j)}\|_1& \hdots
\end{bmatrix}
$}
\end{equation*}
We use the $l_1$ norm to aggregate the prediction, in case that the task is multi-dimensional. 
These differences are used to identify tasks, as $\|\vg_i(\vx^{(j)})-\vy^{(j)}\|_1\approx 0$ for $\vy^{(j)}$ coming from task $i$. Different heads at each layer in the model are used to execute each of the tasks in parallel using the weights from $\text{TF}_i$. In Appendix \ref{app:construction} we construct tasks where an arbitrary function $\vg_i(\vx^{(j)}_l)$ is implemented using $\relu$s for some fixed $l$ that is task-specific. 

\paragraph{Creating task identifiers.} Having the differences between the implemented function and the label, we first use the $\relu$s to clean up the vectors $\vv_k$ so that only the positions in each vector that are associated with a task are maintained and the rest are set to $1$\footnote{This step is not mandatory, but it ensures that we have no values over which we have no control. We leave as future work an error analysis on how these values could affect the task identifiers.}. We thus create the vectors $(\vv'_k)_i =  \|\vg_k(\vx_{1:l-1}^{(j)}) - \vx_l^{(j)}\|_1$ and $(\vv'_*)_*=1$ otherwise. Now we use ReLUs to threshold and create an indicator vectors $\vOne_{\{\|\vg_k(\vx_{1:l-1}^{(j)}) - \vx_l^{(j)}\|_1 \approx 0\}}$ which identify the task, \ie these are task identifiers. Notice that if the task is correctly predicted then the difference should be close to $0$ (up to some error), while if the task is not identified the corresponding value would not be $0$; the rest of the rows would be $1$. We have created one vector for each task, which has $1$ in the position of the corresponding task if the task was identified in the context.\looseness-1
\paragraph{Averaging and task superposition.} As a last step, we average all the task identifiers and place the result in the last column, in which the next prediction will happen. We then use the averaged task identifier to weight the prediction of each task based on it, as in task superposition. If the task has been identified multiple times in the context, it would be assigned a higher weight/probability.

\section{Task superposition through the lens of task vectors} \label{sec:taskvectors}
While in Section \ref{sec:construction} we provide an existential result by constructing a Transformer that performs task superposition and shows that task superposition is well within the expressive power of Transformers, we would like to further investigate how task superposition manifest in pretrained LLMs internally. In this section we explore the underlying mechanisms that LLMs employ during task superposition. In particular, we focus our empirical study on \emph{task vectors} \citep{hendel_-context_2023} where the detailed implementation is in Appendix \ref{app:task_vector_details}. Task vectors are vectors in the embedding space and are found to encode the algorithm that a model internally implements to solve a task given in-context demonstrations.

\begin{figure}
    \begin{subfigure}[b]{0.5\textwidth}
    \centering
    \includegraphics[scale =0.4]{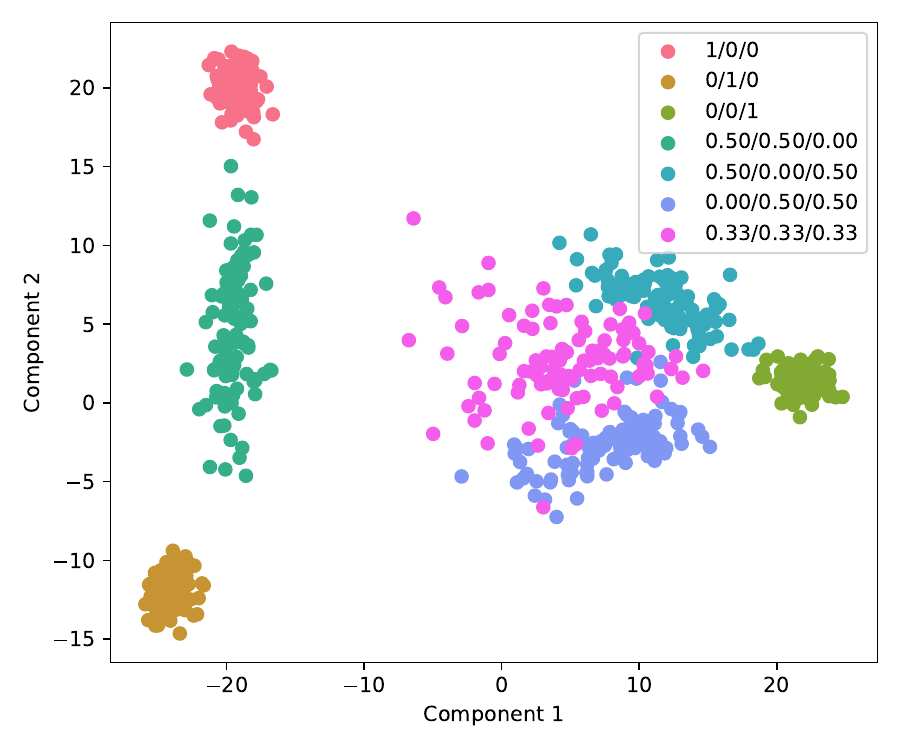}
    \caption{\texttt{copy(op1)} / \texttt{copy(op2)} / \texttt{op1+op2}}
    \label{fig:tv_copy}
    \end{subfigure}
    \begin{subfigure}[b]{0.5\textwidth}
    \centering
    \includegraphics[scale =0.4]{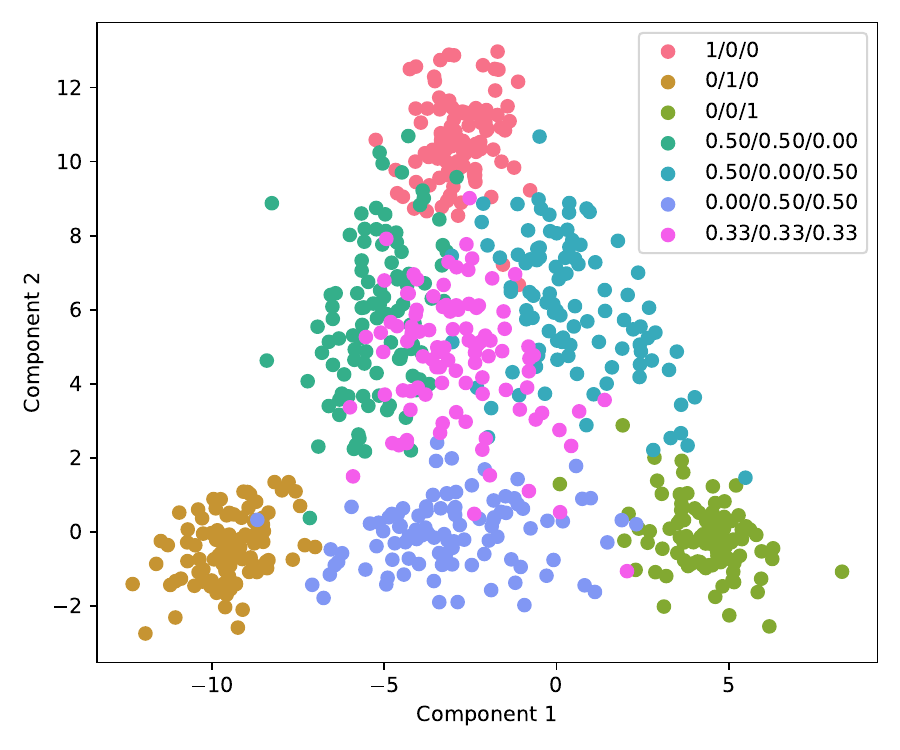}
    \caption{\texttt{to\_fr} / \texttt{to\_de} / \texttt{to\_it}}
    \label{fig:tv_translate}
    \end{subfigure}
\vspace{-10pt}
\caption{Task vectors of Llama-3 8B projected onto two axes chosen by LDA for two sets of tasks: \textbf{(a)} \texttt{copy(op1)}, \texttt{copy(op2)} and \texttt{op1+op2} and \textbf{(b)} \texttt{to\_fr}, \texttt{to\_de} and \texttt{to\_it}. For tasks $t_1, t_2, t_3$, we use ``$\pr(t_1)/\pr(t_2)/\pr(t_3)$'' to denote different levels of task mixtures, e.g., ``0.50/0.50/0.00'' represents the case where the in-context task examples are $50\%$ $t_1$, $50\%$ $t_2$ and $0\%$ $t_3$.}
\label{fig:tv_mixture}
\vspace{-15pt}
\end{figure}

We want to investigate if there is any relation between the task vectors of each individual task and the task vectors of a mixture of task examples in the prompt. To this end, we consider two sets of tasks:

\begin{enumerate}[label=(\alph*)]
    \item \texttt{copy(op1)}, \texttt{copy(op2)} and \texttt{op1+op2} as in Figure \ref{fig:gpt3_demo} (left).
    \item Given a two-digit integer, task \texttt{to\_fr} translates it to French, task \texttt{to\_de} translates it to German and task \texttt{to\_it} translates it to Italian.
\end{enumerate}

For each set of tasks, we collect the task vectors for each individual task and task vectors extracted from prompts that contain examples of different tasks. In Figure~\ref{fig:tv_mixture}, we project task vectors along two axes chosen by linear discriminant analysis (LDA). 

\begin{highlight}
    \paragraph{Finding 3:} 
    \emph{LLMs internally combine task vectors during task superposition.}
\end{highlight}

Interestingly, we observe that the locations of task vectors of a mixture of tasks strongly correlate with the locations of task vectors for each individual task and the in-context task example distribution (the mixture ratio for examples of different tasks). For example, if the prompt includes an equal number of in-context examples from each task, the task vectors are roughly centered in the middle; if the prompt only contains in-context examples of two tasks, then the task vectors roughly lie on the connecting line between task vectors of two individual tasks. We argue that this observation is indicative of the fact that, when prompted with a mixture of in-context task examples, LLMs internally combine task vectors.

As we observe signs that LLMs internally compose task vectors, we want to further investigate whether we can reproduce the task superposition phenomenon by patching in a convex combination of task vectors. For example, for tasks \texttt{copy(op1)} and \texttt{copy(op2)}, we first extract the corresponding task vectors $\text{V}_{\texttt{copy(op1)}} $ and $\text{V}_{\texttt{copy(op2)}}$ on Llama-3 8B using the method described in Appendix \ref{app:task_vector_details}. We then make a convex combination of the two task vectors with parameter $\lambda$ that controls the ratio:
\begin{align*}
    \text{V}_{\text{interpolate},\lambda}=\lambda \cdot \text{V}_{\texttt{copy(op1)}} + (1-\lambda) \cdot \text{V}_{\texttt{copy(op2)}}.
\end{align*}

\begin{highlight}
    \paragraph{Finding 4:} 
    \emph{Convex combinations of task vectors produce task superposition.}
\end{highlight}

\begin{figure}[h!]
    \begin{subfigure}[b]{0.05\textwidth}
        \raisebox{0.9\totalheight}{%
            \begin{tabular}{c}
                \rotatebox{90}{\small interpolation} \\[0.9cm]
                \rotatebox{90}{\small in-context} \\
            \end{tabular}%
        }%
    \end{subfigure}%
    \begin{subfigure}[b]{0.45\textwidth}
        \includegraphics[width=\textwidth]{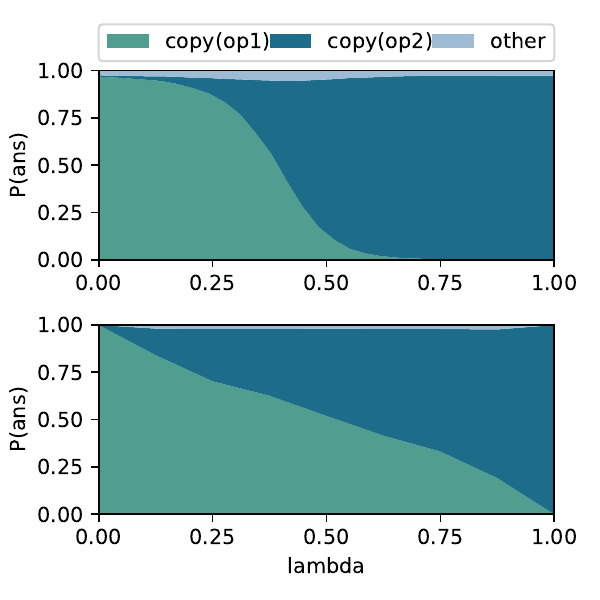}
        \vspace{-7pt}
        \caption{Tasks: \texttt{copy(op1)} and \texttt{copy(op2)}}
        \label{fig:subcopy}
    \end{subfigure}%
    \begin{subfigure}[b]{0.45\textwidth}
        \includegraphics[width=\textwidth]{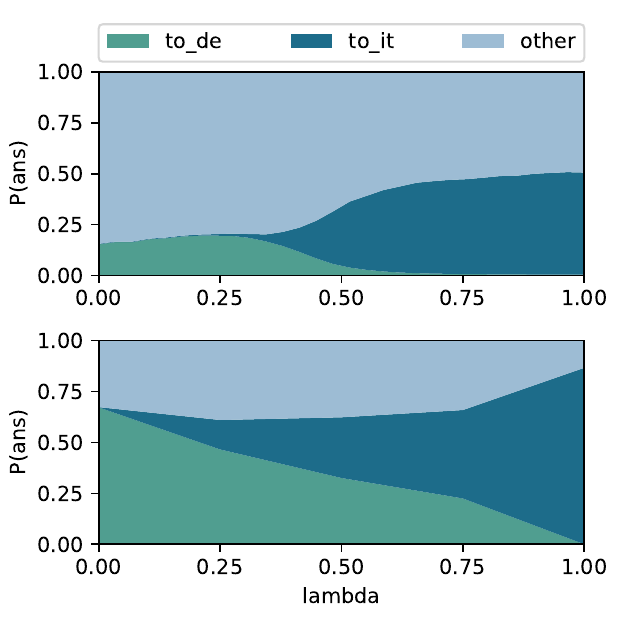}
        \vspace{-7pt}
        \caption{Tasks: translate \texttt{to\_de} and \texttt{to\_it}}
        \label{fig:subtranslate}
    \end{subfigure}
    \vspace{-5pt}
    \caption{On Llama-3 8B, we vary the proportion, $\lambda$, between two tasks and observe how the output probabilities for the correct answers change. The proportion $\lambda$ is varied in two ways: (1) in the top row, we plot the output from patching in a convex combination of task vectors for two tasks. (2) in the bottom row, we plot the output from a mixed proportion of in-context examples for the two tasks. Subplot (a) shows the output probabilities from mixing two copy tasks and (b) shows the probabilities from mixing two translate tasks.
    }
    \label{fig:tv_interpolation}
    \vspace{-15pt}
\end{figure}

For a new query (in this scenario in the form ``\texttt{\{op1\}@\{op2\}$=$}''), we patch the vector $\text{V}_{\text{interpolate},\lambda}$ into the model at the task vector layer. We calculate the model output probabilities that correspond to each task while we vary $\lambda$. For each $\lambda$, we use $100$ different queries and plot the average probabilities in the top row of Figure~\ref{fig:tv_interpolation}. As a comparison, in the bottom rows of Figure~\ref{fig:tv_interpolation}, we plot the corresponding output probabilities when providing the models with prompts containing mixture of task examples where the mixture ratio is controlled by $\lambda$.

In top row of Figure \ref{fig:tv_interpolation}, we observe that patching convex combinations of task vectors into the model produces task superposition. We would also like to point out that in Figure \ref{fig:subtranslate}, although irrelevant outputs sum up to a large probability, the task answers for two tasks \texttt{to\_de} and \texttt{to\_it} in most cases will still be the top-2 answers.

Comparing the top rows and the bottom rows, we can see that top rows (the scenario of interpolating task vectors of individual tasks) have larger probabilities of irrelevant output (category \texttt{other}). Task vector interpolation also produces less of a linear relationship between $\lambda$ and the output probabilities. This shows that while convex combinations of task vectors are sufficient for producing task superposition, the convex combination does not fully explain task superposition. We leave it to future work to investigate other mechanistic explanations of task superposition.

\section{Task superposition capabilities as the model scales}
\label{sec:scale}

\begin{highlight}
    \paragraph{Finding 5:} 
    \emph{Within the same LLM family, bigger models can solve more tasks in parallel and better calibrate to ICL distribution.}
\end{highlight}

We want to further investigate how models' task superposition capabilities change as the model size scales. In particular, we investigate two questions: 1) whether larger models can perform more tasks in-context and 2) whether larger models can align their output distribution more closely with the distribution of task examples provided in the prompt. We chose the Qwen-1.5 model family since it contains several model sizes ranging from 0.5B to 14B parameters. 

We first introduce a quantity which captures the capability of a model to perform multiple tasks. Given a prompt that contains examples of $K$ tasks, we define $r$ to be the number of these tasks whose correct answers appear among the model's top-$K$ most likely outputs. Note that $r \leq K$.

To see how close the model align the output distribution with the distribution of task examples, we use KL-divergence defined below:
\begin{equation}
    \text{KL}(\mathcal{P}||\cD)=\sum_{x\in X}\mathcal{P}(x)\log\left(\frac{\mathcal{P}(x)}{\cD(x)}\right),
\end{equation}
where $\mathcal{P}$ is the models' probabilities on the outputs when correctly performing each task on the query and $\cD$ is the in-context task example distribution. For example the prompt in Figure \ref{fig:llama3_demo} (left) gives $\mathcal{P}=[0.5217, 0.1316, 0.1110, 0.2169, ...]$ and $\cD=[0.25, 0.25, 0.25, 0.25, 0, ...]$.

We consider the setting of $K=6$ different tasks: given an input of the form ``\texttt{\{num\}$\rightarrow$}'' where $\texttt{num}$ is a two-digit integer, we consider $6$ tasks that output (1) \texttt{num} itself, (2) negation of \texttt{num}, (3) $\texttt{num} + 1$, (4) $\texttt{num}-1$, (5) $\texttt{num}\times 2$ and (6) $\texttt{num}^2$.

We choose the number of in-context examples $m=60$ (each task has $10$ examples) and configure the prompt with three different in-context task example distributions $\cD_1, \cD_2$ and $\cD_3$. In particular, $\cD_1$ is the uniform distribution, $\cD_2$ has probability $0.5$ on the third task and $0.1$ on other tasks, and $\cD_3$ is a distribution with probabilities alternating between $0.25$ and $0.083$.

For each in-context task examples distribution $\cD_i$, we generate $100$ prompts and for each prompt we calculate the probabilities of outputs when correctly performing each task. The average values of $r$ and KL-divergence under three distributions are shown in Figure \ref{fig:r_and_kl}. 

\begin{figure}[h!]
    \vspace{-6pt}
    \centering
    \begin{subfigure}[t]{0.49\textwidth}
        \centering
        \includegraphics[width=1\linewidth]{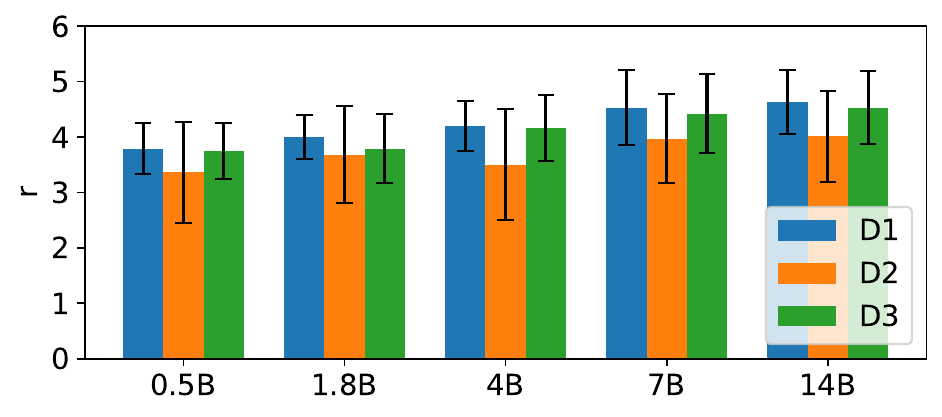}
        \vspace{-7pt}
        \captionsetup{width=0.9\textwidth}
        \caption{\centering $r$ (the number of tasks whose correct answers appear in top-$K$ most likely outputs).}
        \vspace{-7pt}
        \label{fig:r_val}
    \end{subfigure}
    \begin{subfigure}[t]{0.49\textwidth}
        \centering
        \includegraphics[width=1\linewidth]{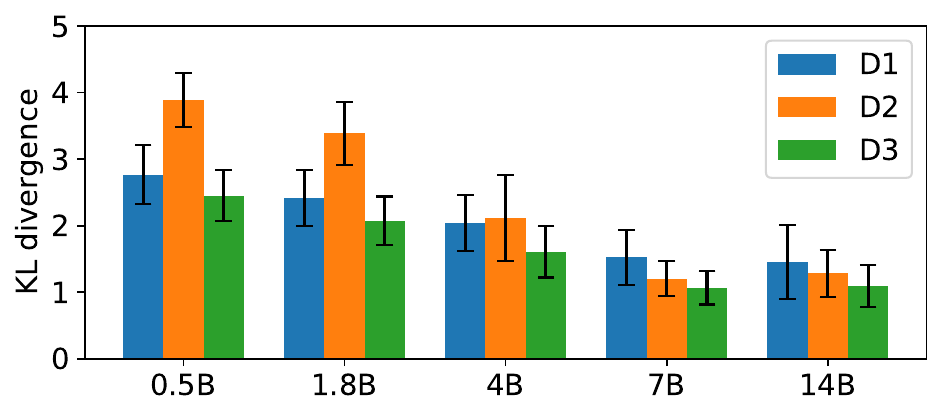}
        \vspace{-7pt}
        \caption{KL divergence.}
        \vspace{-7pt}
        \label{fig:kl_div}
    \end{subfigure}
    \caption{(a) Average number of tasks completed, $r$, and (b) KL divergence for Qwen-1.5 model family under ICL distributions $\cD_1, \cD_2$ and $\cD_3$ where $\cD_1$ is the uniform distribution, $\cD_2$ has probability $0.5$ on the third task and $0.1$ on other tasks, and $\cD_3$ is a distribution with probabilities alternating between $0.25$ and $0.083$.}
    \label{fig:r_and_kl}
    \vspace{-6pt}
\end{figure}

In Figure \ref{fig:r_val}, we can observe that bigger models have higher $r$ values (except for task distribution $\cD_2$, 4B model has slightly lower $r$ than that of the 1.8B model). This shows bigger models will have more correct answers of tasks show up in their top-$K$ probable outputs and therefore they can solve more tasks at the same time. In Figure \ref{fig:kl_div}, we can see that for larger models like Qwen-1.5 7B and Qwen-1.5 14B, the KL-divergence values are small, and for each model, the differences between KL-divergence values under in-context task example distributions $\cD_1$, $\cD_2$ and $\cD_3$ are small. This indicates that bigger models can better calibrate their output distribution to the in-context task example distribution.

\section{Limitations and future directions}

One limitation of our work is the current gap between the demonstrated capability of LLMs to perform task superposition and its practical application in real-world scenarios. While we have shown that LLMs possess the capacity to execute multiple tasks simultaneously, conventional decoding algorithms are not equipped to fully leverage this capability. This limitation stems from what we term "generation collapse," a phenomenon where, after the first token is generated, the model tends to converge on predicting tokens for a single task, effectively negating its ability for multi-task execution.

This collapse presents a substantial challenge in harnessing the full power of task superposition. It highlights a critical area for future research: developing decoding strategies that can maintain the model's multi-task state throughout the generation process.
Recent work by \citet{shen2024superposed} offers some hope that this direction may be fruitful, by proposing a ``superposed decoding'' algorithm. Their method efficiently generates multiple streams of tokens from a single inference pass by utilizing superposed token embeddings. While this approach represents a significant step forward, it also highlights the potential for further innovation in this area.

\section{Conclusion} \label{conclusion}

We report on the discovery of task superposition, which is the ability of LLMs to simultaneously solve distinct tasks from in-context examples. Task superposition is present in a variety of pretrained models, and becomes more accurate at predicting the distribution of tasks as the model size increases. We also find evidence that while displaying task superposition, models internally mix the task vectors of each individual task. We hope that our findings will contribute to understanding in-context learning mechanisms and enhance our knowledge of LLMs overall.

\bibliography{iclr2025_conference,bibliography,zotero_citations}
\bibliographystyle{iclr2025_conference}

\newpage
\appendix

\section{Notations}
\begin{table*}[h]
    \centering
    \begin{tabular}{cc}
        \hline
        \textbf{Notation} & \textbf{Description} \\
        \hline
        $K$ & Number of tasks \\
        $l$ & Length of a task's output \\
        $\ell$ & layer $\ell$ for a model \\
        $m$ & Number of in-context examples \\
        $n$ & Length of each in-context example \\
        $\cV$ & Token vocabulary \\
        $\vg_i(\cdot)$ & Operation performed by Task $i$ \\ 
        $\vx^{(j)}$ & Data for example $j$ \\
        $\vy^{(j)}$ & Label for example $j$ \\
        $\vs_m$ & $m$ in-context examples \\
        $\vf(\cdot)$ & Model (predictor)\\
        $\vp$  & Positional encodings\\
        \hline
    \end{tabular}
    \label{tab:notation}
\end{table*}
\newpage

\section{Implementation details on calculating probabilities}
\label{app:prob_details}
In this section we provide details on how we calculate probabilities of different outputs given a prompt in our setting.

\paragraph{Notations.} Let $\mathcal{V}$ be the token vocabulary, $M$ be an LLM, $T$ be the tokenizer. We use ``...'' to represent a string, \texttt{<...>} to represent a single token where the content within the angle brackets is an integer representing token's index in vocabulary. For example, token \texttt{<266>} corresponds to ``at''. We use [\texttt{<...>, ..., \texttt{<...>}}] to represent a sequence of tokens. Given a tokenizer, we use two functions \texttt{tok}($\cdot$) and \texttt{detok}($\cdot$) to tokenize strings and detokenize tokens. For example $\texttt{tok}(\text{``superposition''})=\text{[\texttt{<9712>},\texttt{<3571>}]}$ and $\texttt{detok}\text{([\texttt{<16>},\texttt{<10>},\texttt{<16>},\texttt{<28>},\texttt{<17>}}])=\text{``1+1=2''}$.

In our in-context learning setting, an input string consists of in-context examples (separated by the delimiter ``\texttt{\textbackslash n}'') and a query. For example, an example prompt can be ``1+1=2\texttt{\textbackslash n}2+2=4\texttt{\textbackslash n}3+3=''.

We view an LLM as a next-token predictor and there is a corresponding $\pr(\cdot|\cdot)$ such that given a sequence of tokens $[v_1,...,v_M]$ where $v_j\in \mathcal{V}$, $\pr(u\mid[v_1,...,v_M])$ measures the probability of the next token being $u$ where $u\in \mathcal{V}$.

\paragraph{Measuring the probabilities of task answers.} Let $I$ be the input prompt. For example, in the example in Figure \ref{fig:llama3_demo} (left), the prompt is ``11+26\texttt{-\textgreater}37\texttt{\textbackslash n}33+13\texttt{-\textgreater}quarante-six\texttt{\textbackslash n}
...30+25\texttt{-\textgreater}fifty-five\texttt{\textbackslash n}91+83\texttt{-\textgreater}''. We consider four tasks: 1) numerical addition, 2) addition in English, 3) addition in French and 4) addition in Spanish. The corresponding task answers (the output of correctly performing task on the query) are ``174'', ``one hundred and seventy-four'', ``cent soixante-quatorze'' and ``ciento setenta y cuatro'', respectively. We want to measure the probability of each task answer.

Let $o$ be a task answer in string. Let $[v_1,...,v_M]:=\texttt{tok}(I)$ and let $[u_1,...,u_N]:=\texttt{tok}(o)$. Then the probability of the task answer $o$ given prompt $I$ can be calculated as 
\begin{equation}
\label{eq:naive_calc}
\pr(u_1\mid[v_1,...,v_M])\prod_{j=2}^{N}\pr(u_j\mid [v_1,...,v_M,u_1,...,u_{j-1}]).
\end{equation}

\newpage

\section{Implementation details on task vectors}
\label{app:task_vector_details}
We use the task vector definition from \citet{hendel_-context_2023}. For example, for task \texttt{copy(op1)} in Figure \ref{fig:gpt3_demo} (left), the procedure to collect the task vector consists of
\begin{enumerate}
    \item Collect a dataset of $100$ ICL sample prompts. Each prompt consists of $m=60$ in-context examples of a particular task and a query $\vx^{(m+1)}$. Each task example $(\vx^{(j)}, \vy^{(j)})$ follows the form ``\texttt{\{op1\}@\{op2\}$=$\{op1\}}'', where $\vx^{(j)}$ has the form ``\texttt{\{op1\}@\{op2\}}$=$'' and $\vy^{(j)}$ is performing task \texttt{copy(op1)} on $\vx^{(j)}$, namely \texttt{op1}.
    \item  For each prompt $\equiv\vs=[\vx^{(1)}, \vy^{(1)},...,\vx^{(m)}, \vy^{(m)}, \vx^{(m+1)}]$ in the dataset, we feed $\vs$ into the transformer model $\vf$, and extract the feature (which is a vector) at the last ``\texttt{$=$}'' token in layer $\ell$. Call this vector $\vf(\vs;\ell)$. Then we average  $\vf(\vs;\ell)$ across all prompt $\vs$ to get $\vv(\ell)$ for layer $\ell$.
    \item Now for each layer $\ell$ we have a vector $\vv(\ell)$. We run a forward pass with one query $\vx$ in the form ``\texttt{\{op1\}@\{op2\}$=$}'' and we patch in $\vv(\ell)$ at the ``\texttt{$=$}'' token position in layer $\ell$, simulating the effect of a complete context. We repeat this process $100$ times for different query $\vx$ and get an accuracy $acc_\ell$ of performing task \texttt{copy(op1)} with vector $\vv(\ell)$.
    \item The task vector layer $\ell^*$ is selected by $$\ell^*=\arg\max_{\ell} acc_\ell,$$ and we define the task vector $\text{V}_{\texttt{copy(op1)}}:=\vv = \vv(\ell^*)$.
\end{enumerate}
Here we record the task vector layer where task vectors are extracted in Section \ref{sec:taskvectors}.
\begin{table*}[h]
    \centering
    \begin{tabular}{|c|c|}
        \hline
        \textbf{Task} & \textbf{Task vector layer} \\
        \hline
        \texttt{copy(op1)}, \texttt{copy(op2)}, \texttt{op1+op2} & 14 \\
        \hline
        \texttt{to\_de(op1)}, \texttt{to\_fr(op1)}, \texttt{to\_it(op1)} & 19 \\
        \hline
    \end{tabular}
    \caption{Task vector layer for various tasks considered in Section \ref{sec:taskvectors}.}
    \label{tab:task_vec_details}
\end{table*}

\newpage

\section{Construction displaying superposition}\label{app:construction}

In this section we construct a Transformer that is performing superposition of multiple tasks at inference. For this purpose, we first construct a Transformer that copies from $n$-tuple in-context examples the $i$-th one, as well as any function using the ReLU layers. We then create indicator vectors, for each task, which show whether a specific task is present in-context or not. As a last step, we combine these indicator vectors to create the superposition of different tasks. Notice that using the parallel heads of the transformer architecture we can process each task independently until the last step in which the predictions are combined.

\subsection{Overview}
Here we provide a brief overview of  how the construction is implemented, while latter we provide the corresponding details.

  \paragraph{Prediction based on multiple tasks.} Assume that we are given $m$ in-context samples $(\vx_1^{(j)},\hdots,\vx_{n-2}^{(j)}, \text{`}\eq\text{'}, \vy^{(j)})_{j=1}^m$  where `$\eq$' represents a specific value used only for preceding the label, and a set of $k$ different Transformers $\text{TF}_i$ which can implement the $T$ different desired tasks, where each deterministic task is denoted as $\vg_i(\vx^{(j)})$ with $i\in[k]$ and $j\in[m]$, \textit{i.e.} $\vy^{(j)}=\vg_i(\vx^{(j)})$ for some task $i$ dependent on sample $j$. Using the weights of each $\text{TF}_i$, we can compute outputs of the following form:
\begin{equation*}
    \begin{bmatrix}
        \hdots & \vx_1^{(j)}  & \hdots & \vx_{n-2}^{(j)} & \eq & \vy^{(j)} & \hdots\\
        \hdots  & \vZero & \hdots & \vZero & \vZero & \vZero & \hdots \\
          & \vdots &  & \vdots & \vdots & \vdots & \\
        \hdots & \vZero & \hdots & \vZero & \vZero & \vZero & \hdots
\end{bmatrix} \to \begin{bmatrix}
        \hdots & \vx_1^{(j)} & \hdots & \vx_{n-2}^{(j)} & \eq & \vy^{(j)} & \hdots\\
        \hdots & \vZero & \hdots & \vZero & \vZero &\|\vg_1(\vx^{(j)}) - \vy^{(j)}\|_1& \hdots \\
        & \vdots & & \vdots & \vdots & \vdots & \\
        \hdots & \vZero & \hdots & \vZero  & \vZero &\|\vg_T(\vx^{(j)}) - \vy^{(j)}\|_1& \hdots
\end{bmatrix}
\end{equation*}
We use the $l_1$ norm to aggregate the prediction, in case that the task is multi-dimensional. 
These differences are used to identify tasks, as $\|\vg_i(\vx^{(j)})-\vy^{(j)}\|_1\approx 0$ for $\vy^{(j)}$ coming from task $i$. Different heads at each layer in the model are used to execute each of the tasks in parallel using the weights from $\text{TF}_i$. In Appendix \ref{app:construction} we construct tasks where an arbitrary function $\vg_i(\vx^{(j)}_l)$ is implemented using $\relu$s for some fixed $l$ that is task-specific. 

\paragraph{Creating task identifiers.} Having the differences between the implemented function and the label, we first use the $\relu$s to clean up the vectors $\vv_k$ so that only the positions in each vector that are associated with a task are maintained and the rest are set to $1$\footnote{This step is not mandatory, but it ensures that we have no values over which we have no control. We leave as future work an error analysis on how these values could affect the task identifiers}. We thus create the vectors $(\vv'_k)_i =  \|\vg_k(\vx_{1:l-1}^{(j)}) - \vx_l^{(j)}\|_1$ and $(\vv'_*)_*=1$ otherwise. Now we use ReLUs to threshold and create an indicator vectors $\vOne_{\{\|\vg_k(\vx_{1:l-1}^{(j)}) - \vx_l^{(j)}\|_1 \approx 0\}}$ which identify the task, \textit{i.e.} these are task identifiers. Notice that if the task is correctly predicted then the difference should be close to $0$ (up to some error), while if the task is not identified the corresponding value would not be $0$; the rest of the rows would be $1$. We have created one vector for each task, which has $1$ in the position of the corresponding task if the task was identified in the context.\looseness-1
\paragraph{Averaging and task superposition.} As a last step, we average all the task identifiers and place the result in the last column, in which the next prediction will happen. We then use the averaged task identifier to weight the prediction of each task based on it, as in task superposition. If the task has been identified multiple times in the context, it would be assigned a higher weight/probability.

\subsection{Task Identification}
The first task for performing task superposition based on in-context examples is to define a set of tasks that the model is able to implement. 

First, the outputs of tasks need to be identified.

\begin{lemma}\label{lem:ident}
    Consider the following input 
    \begin{equation*}
    \mX = \begin{bmatrix}
        \vx_1^{(1)} & \hdots &\vy^{(j-1)}& \vx_1^{(j)}  & \hdots & \vx_{n-2}^{(j)} &\eq& \vy^{(j)} & \vx_1^{(j+1)} & \hdots \\
        0& \hdots&0  & 0 & \hdots & 0& 1 & 0 & 0 & \hdots \\
        \vZero & \hdots&\vZero  & \vZero & \hdots & \vZero& \vZero & \vZero & \vZero & \hdots
    \end{bmatrix}\;,
\end{equation*} 
where $\vx_i^{(j)}\in\mathbb{R}^{d-1} $ before the positional encodings are added, with one additional dimension that represents if the symbol is an `equals' symbol. Then, a 1-layer transformer with a single attention head and embedding dimension $\cO(d + \log(mn))$ can output
\begin{equation*}
    \mX = \begin{bmatrix}
        \vx_1^{(1)} & \hdots &\vy^{(j-1)}& \vx_1^{(j)}  & \hdots & \vx_{n-2}^{(j)} &\eq& \vy^{(j)} & \vx_1^{(j+1)} & \hdots \\
        0& \hdots& 1 & 0 & \hdots & 0 & 0 & 1 & 0 & \hdots \\
    \end{bmatrix}
\end{equation*}
\end{lemma}
\begin{proof}
    With positional encodings appended, let the input have the following structure:

    \begin{equation}
        \mX = \begin{bmatrix}
            \vx_1^{(1)} & \hdots & \vx_1^{(j)} & \vx_2^{(j)} & \hdots & \vx_{n-2}^{(j)} & \eq & \vy^{(j)} & \vx_1^{(j+1)} & \hdots \\
            0& \hdots&0  & 0 & \hdots & 0& 1 & 0 & 0 & \hdots \\
            \vp_{n+1} & \hdots & \vp_{jn+1} & \vp_{jn+2} & \hdots & \vp_{jn+n-2} & \vp_{jn+n-1} & \vp_{jn+n} & \vp_{(j+1)n+1} & \hdots \\
            \vp_{n} & \hdots & \vp_{jn} & \vp_{jn+1} & \hdots & \vp_{jn+n-3} & \vp_{jn+n-2} & \vp_{jn+n-1} & \vp_{(j+1)n} & \hdots \\
        \end{bmatrix}
    \end{equation}   
    To rotate the second row one position to the right, use the following matrices.

    \begin{align*}
            \mW_Q &= \begin{bmatrix}
            \mZero & 0 & \mZero & \iden
        \end{bmatrix} \\
        \mW_K &= \begin{bmatrix}
            \mZero & 0 & C\iden & \mZero
        \end{bmatrix} \\
        \mW_V &= \begin{bmatrix}
            \mZero & \mZero & \mZero & \mZero \\
            \mZero &      1 & \mZero & \mZero \\
            \mZero & \mZero & \mZero & \mZero \\
            \mZero & \mZero & \mZero & \mZero \\
        \end{bmatrix}
    \end{align*}
    The pair $\mW_Q$ and $\mW_K$ attend tokens to the token directly to the right. The value matrix simply filters only the second row in-place. A second head can used to clear the original $1$s, resulting in

    \begin{equation}
        \mX = \begin{bmatrix}
            \vx_1^{(1)} & \hdots & \vx_1^{(j)} & \vx_2^{(j)} & \hdots & \vx_{n-2}^{(j)} & \eq & \vy^{(j)} & \vx_1^{(j+1)} & \hdots \\
            0 & \hdots& 0 & 0 & \hdots & 0& 0 & 1 & 0 & \hdots \\
            \vp_{n+1} & \hdots & \vp_{jn+1} & \vp_{jn+2} & \hdots & \vp_{jn+n-2} & \vp_{jn+n-1} & \vp_{jn+n} & \vp_{(j+1)n+1} & \hdots \\
            \vp_{n} & \hdots & \vp_{jn} & \vp_{jn+1} & \hdots & \vp_{jn+n-3} & \vp_{jn+n-2} & \vp_{jn+n-1} & \vp_{(j+1)n} & \hdots \\
        \end{bmatrix}\;,
    \end{equation}
    as desired.
\end{proof}

\paragraph{Implementation of functions.} To illustrate a set of operations that could be implemented with a transformer, we consider approximating functions as sums of ReLUs;  we use a result from \citet{bai2023transformers}, which we present below.
\begin{definition}[Definition 12 in \cite{bai2023transformers}]\label{def:approx}
 A function $g:\R^k\to\R$ is $(\epsilon,R,M,C)$-approximable by sum of ReLUs, if there exists an "(M,C)-sum of ReLUs" function
 \begin{equation*}
     f_{M,C}(\vx) =\sum_{m=1}^M c_m\relu(\va_m^\top[\vx;1]) \text{ with } \sum_{m=1}^M\abs{c_m} \leq C, \; \max_{m\in[M]}\norm{\va_m}_1 \leq 1, \; \va_m\in\R^{k+1},\;c_m\in\R
 \end{equation*}
 such that $\sup_{\vx\in[-R,R]^k]}\abs{g(\vx)-f_{(M,C)}(\vx)}\leq \epsilon$.
\end{definition}
\begin{definition}[Definition A.1 in \cite{bai2023transformers}]
    We say a function $g:\R^k\to\R$ is $(R,C_l)$-smooth if for $s=\left \lceil  (k-1)/2 \right \rceil+2$, $g\in C^2$\footnote{$C^i$ denotes that a function is $i$ times differentiable with continuous $i$-th derivative.} on $[-R,R]^k$ and 
    \begin{equation*}
        \sup_{\vx\in[-R,R]^k}\norm{ \nabla^{i}g(\vx)}_{\infty} =\sup_{\vx\in[-R,R]^k}\max_{j1,\hdots,j_i\in[k]}\abs{\partial_{x_{j1},\hdots,x_{ji}}g(\vx)}\leq L_i
    \end{equation*}
for all $i=0,1,2$ with $\max_{0\leq i\leq s} L_iR^i \leq C_l$.
\end{definition}
\begin{proposition}[Proposition A.1 in \cite{bai2023transformers}]\label{prop:relus}
    For any $\epsilon >0$, $R\geq 1$, $C_l >0$, we have that: Any $(R,C_l)$-smooth function, $g:\R\to\R$ is $(\epsilon,R,M,C)$-approximable by sum of ReLUs (Definition \ref{def:approx}) with $M\leq C(k)C_l^2\log(1+C_l\epsilon)/\epsilon^2$.
\end{proposition}

\begin{lemma}\label{lem:generaltask}
    For any function $g:\R^k\to \R$ that is $(R,C_l)$-smooth, there exists a transformer 
    with two layers, one head and width $\cO(\log(n) + d)$, where $d$  satisfies the requirements of Prop. \ref{prop:relus}, such that given as input 
    \begin{equation*}
    \mX = \begin{bmatrix}
        \vx_1^{(1)} & \hdots &\vy^{(j-1)}& \vx_1^{(j)}  & \hdots & \vx_{n-2}^{(j)} &\eq& \vy^{(j)} & \vx_1^{(j+1)} & \hdots \\
        0& \hdots& 1 & 0 & \hdots & 0 & 0 & 1 & 0 & \hdots \\
        \vZero & \hdots&\vZero  & \vZero & \hdots & \vZero& \vZero & \vZero & \vZero & \hdots
    \end{bmatrix}\;,
    \end{equation*}
   it outputs
   \begin{equation*}
        \mX = \begin{bmatrix}   
            \vx_1^{(1)} & \hdots &\vy^{(j-1)}& \vx_1^{(j)}  & \hdots & \vx_{n-2}^{(j)} &\eq& \vy^{(j)} & \vx_1^{(j+1)} & \hdots \\
            * & \hdots& * & * & \hdots & * & * & * & * & \hdots \\
            \vZero & \hdots& \tilde{g}(\vx_i^{(j-1)})-\vy^{(j-1)}  & * & \hdots & * &* & \tilde{g}(\vx_i^{(j)})-\vy^{(j)}   & * & \hdots 
        \end{bmatrix}
    \end{equation*}
    where $\abs{\tilde{g}(x) -g(x)}\leq \epsilon$ and for some $i \in [1,\hdots,n-2]$.
\end{lemma}
\begin{proof}
    We consider that the positional encodings are added in the input and we have
    \begin{equation}
        \mX = \begin{bmatrix}
            \vx_1^{(1)} & \hdots & \vx_1^{(j)} & \vx_2^{(j)} & \hdots & \vx_{n-2}^{(j)} & \eq & \vy^{(j)} & \vx_1^{(j+1)} & \hdots \\
            0& \hdots&0  & 0 & \hdots & 0 & 0 & 1 & 0 & \hdots \\
            \vZero & \hdots & \vZero & \vZero & \hdots & \vZero & \vZero & \vZero & \vZero & \hdots \\
            1 & \hdots & 1 & 1 & \hdots & 1 & 1 & 1 & 1 & \hdots \\
            \vp_{n+1} & \hdots & \vp_{jn+1} & \vp_{jn+2} & \hdots & \vp_{jn+n-2} & \vp_{jn+n-1} & \vp_{jn+n} & \vp_{(j+1)n+1} & \hdots \\
            \vp_{n+1-s} & \hdots & \vp_{jn+1-s} & \vp_{jn+2-s} & \hdots & \vp_{jn+n-2-s} & \vp_{jn+n-1-s} & \vp_{jn+n-s} & \vp_{(j+1)n+1-s} & \hdots \\
        \end{bmatrix}
    \end{equation}
    where we fix some positional encodings $\vp_k$ where $\vp_k^\top \vp_k$ is larger than $\vp_k^\top \vp_l$ by some threshold for $k \neq l$. The encodings used here are the binary representations of $k \in \{-1,1\}^{\log(mn)}$. Further, we consider $1$s in the positions with the results of the task to differentiate the context of the task and the result of the task. Define $s=n-i$, the distance between the result and the associated value in the context.
    
    In the first layer, we use the MLP's to create $\tilde{g}$ according to Proposition \ref{prop:relus}
    \begin{equation}
        \mX = \begin{bmatrix}
            \vx_1^{(1)} & \hdots & \vx_1^{(j)} & \vx_2^{(j)} & \hdots & \vx_{n-2}^{(j)} & \eq & \vy^{(j)} & \vx_1^{(j+1)} & \hdots \\
            0& \hdots&0  & 0 & \hdots & 0 & 0 & 1 & 0 & \hdots \\
            \tilde{g}(\vx_1^{(1)}) & \hdots & \tilde{g}(\vx_1^{(j)}) & \tilde{g}(\vx_2^{(j)}) & \hdots & \tilde{g}(\vx_{n-2}^{(j)}) & \tilde{g}(\eq) & \tilde{g}(\vy^{(j)}) & \tilde{g}(\vx_1^{(j+1)}) & \hdots \\
            \vZero & \hdots & \vZero & \vZero & \hdots & \vZero & \vZero & \vZero & \vZero & \hdots \\
            1 & \hdots & 1 & 1 & \hdots & 1 & 1 & 1 & 1 & \hdots \\
            \vp_{n+1} & \hdots & \vp_{jn+1} & \vp_{jn+2} & \hdots & \vp_{jn+n-2} & \vp_{jn+n-1} & \vp_{jn+n} & \vp_{(j+1)n+1} & \hdots \\
            \vp_{n+1-s} & \hdots & \vp_{jn+1-s} & \vp_{jn+2-s} & \hdots & \vp_{jn+n-2-s} & \vp_{jn+n-1-s} & \vp_{jn+n-s} & \vp_{(j+1)n+1-s} & \hdots \\
        \end{bmatrix}
    \end{equation}

    The next operation is a shift of the sequence of $\tilde{g}(\cdot)$'s to the right by $s$. This will align the desired output $\tilde{g}(\vx_i^{(j)})$ with the observed output $\vy^{(j)}$. Consider the following weight matrices

    \begin{align}
        \mW_Q &= \begin{bmatrix}
            \hdots & 0 & \mZero & \iden
        \end{bmatrix} \\
        \mW_K &= \begin{bmatrix}
            \hdots & 0 & C\iden & \mZero
        \end{bmatrix} \\
        \mW_V &= \begin{bmatrix}
            \mZero & \mZero & \mZero & \hdots & \mZero \\
            \mZero & \mZero & \mZero & \hdots & \mZero \\
            \mZero & \mZero & \iden  & \hdots & \mZero \\
            \vdots & \vdots & \vdots &        & \vdots \\
            \mZero & \mZero & \mZero & \hdots & \mZero \\
        \end{bmatrix}
    \end{align}
    
    for some large constant $C$ to decrease error from the softmax attending to the incorrect tokens. This produces (within a small error induced by using a softmax)
    
    \begin{align}
        (\mX^\top \mW_K^\top \mW_Q \mX)_{i,j} &= \vp_{n+i}^\top \vp_{n-s+j} \\ 
        \sigma_S(\mX^\top \mW_K^\top \mW_Q X)_{i,j} &= \bOne_{\{n+i = n-s+j\}} = \bOne_{\{i = j-s\}} \\
        \mW_V \mX &= \begin{bmatrix}
            \hdots & \vZero & \vZero & \hdots & \vZero & \vZero & \vZero & \hdots\\
            \hdots & 0 & 0 & \hdots & 0 & 0 & 0 & \hdots \\
            \hdots & \tilde{g}(\vx_1^{(j)}) & \tilde{g}(\vx_2^{(j)}) & \hdots & \tilde{g}(\vx_{n-2}^{(j)}) & \tilde{g}(\eq) & \tilde{g}(\vy^{(j)}) & \hdots \\
            & \vdots & \vdots &        & \vdots & \vdots & \vdots &  \\
            \hdots & \vZero & \vZero & \hdots & \vZero & \vZero & \vZero & \hdots \\
        \end{bmatrix}
    \end{align}
    
    \begin{align}
         \mW_V \mX \sigma_S(\mX^\top \mW_K^\top \mW_Q \mX) &= \begin{bmatrix}
            \hdots & \vZero & \vZero & \hdots & \vZero & \vZero & \vZero & \hdots\\
            \hdots & 0 & 0 & \hdots & 0 & 0 & 0 & \hdots \\
            \hdots & * & * & \hdots & * & * & \tilde{g}(\vx_i^{(j)}) & \hdots \\
            & \vdots & \vdots &        & \vdots & \vdots & \vdots &  \\
            \hdots & \vZero & \vZero & \hdots & \vZero & \vZero & \vZero & \hdots \\
        \end{bmatrix} \\
        \mX + \mW_V \mX \sigma_S(\mX^\top \mW_K^\top \mW_Q \mX) &= \begin{bmatrix}
            \hdots & \vx_1^{(j)} & \vx_2^{(j)} & \hdots & \vx_{n-1}^{(j)} & \eq & \vy^{(j)} & \hdots\\
            \hdots & 0 & 0 & \hdots & 0 & 0 & 1 & \hdots \\
            \hdots & * & * & \hdots & * & * & \tilde{g}(\vx_i^{(j)}) & \hdots\\
            \hdots & \vZero & \vZero & \hdots & \vZero & \vZero & \vZero & \hdots\\
            \hdots & 1 & 1 & \hdots & 1 & 1 & 1 & \hdots \\
            \hdots & \vp_{jn+1} & \vp_{jn+2} & \hdots & \vp_{jn+n-2} & \vp_{jn+n-1} & \vp_{jn+n} & \hdots \\
            \hdots & \vp_{jn+1-s} & \vp_{jn+2-s} & \hdots & \vp_{jn+n-2-s} & \vp_{jn+n-1-s} & \vp_{jn+n-s}& \hdots \\
        \end{bmatrix} \\
    \end{align}
    
    Each matrix above only shows the slice that contains the $j$-th in-context example. This is repeated for each of the other in-context examples.

    As a final step with an MLP, subtract row 1 from row 3 to achieve the following output:
    \begin{equation}
        \begin{bmatrix}
            \hdots & \vx_1^{(j)} & \vx_2^{(j)} & \hdots & \vx_{n-1}^{(j)} & \eq & \vy^{(j)} & \hdots\\
            \hdots & 0 & 0 & \hdots & 0 & 0 & 1 & \hdots \\
            \hdots & * & * & \hdots & * & * & \tilde{g}(\vx_i^{(j)}) - \vy^{(j)} & \hdots\\
            \hdots & \vZero & \vZero & \hdots & \vZero & \vZero & \vZero & \hdots\\
            \hdots & 1 & 1 & \hdots & 1 & 1 & 1 & \hdots \\
            \hdots & \vp_{jn+1} & \vp_{jn+2} & \hdots & \vp_{jn+n-2} & \vp_{jn+n-1} & \vp_{jn+n} & \hdots \\
            \hdots & \vp_{jn+1-s} & \vp_{jn+2-s} & \hdots & \vp_{jn+n-2-s} & \vp_{jn+n-1-s} & \vp_{jn+n-s}& \hdots \\
        \end{bmatrix} \\
    \end{equation}
    
\end{proof}

\paragraph{Copy Tasks} As has been experimentally investigated, the situation where a specific position within the context is copied as the label can be easily implemented by setting $g(\vx) = \vx$. The dependence on the subscript $i$ within the construction is what allows the position copied to vary.

\subsubsection{Identifying if Task's Output Matches the In-context Example}

\begin{lemma}\label{lem:context}
    A three layer transformer with ReLU MLPs and embedding dimension $\cO(d + \log(mn))$ can calculate the proportion of in context examples that come from a specific task, where $m$ is the number of in-context examples, each of length $n$ and dimension $d$.
\end{lemma}
\begin{proof}
We now have a matrix of the following form.

\begin{equation}
    \begin{bmatrix}
        \hdots & \vx_1^{(j)} & \vx_2^{(j)} & \hdots & \eq & \vy^{(j)} & \hdots\\
        \hdots & 0 & 0 & \hdots & 0 & 1 & \hdots \\
        \hdots & * & * & \hdots & * & f(\vx_\cdot^{(j)})-\vy^{(j)}  & \hdots\\
        \hdots & \vZero & \vZero & \hdots & \vZero & \vZero & \hdots\\
        \hdots & 1 & 1 & \hdots & 1 & 1 & \hdots \\
        \hdots & \vp_{jn+1} & \vp_{jn+2} & \hdots & \vp_{jn+n-1} & \vp_{jn+n} & \hdots \\
        \hdots & \vp_{jn+1-s} & \vp_{jn+2-s} & \hdots & \vp_{jn+n-1-s} & \vp_{jn+n-s}& \hdots \\
    \end{bmatrix} 
\end{equation}

If the task is correct, than $f(\vx_\cdot^{(j)})-\vy^{(j)} \approx \vZero$, with some small error coming from softmaxs and function approximation error. First, we find the $L1$-norm of $f(\vx_\cdot^{(j)})-\vy^{(j)}$ using an MLP. For calculating $\| \vz \|_1$ for arbitrary $\vz$, we can use

\begin{equation}
    \|\vz\|_1 = \sum_{i=1}^d \relu(\vz_i) - \relu(-\vz_i)
\end{equation}

which can be done in a single 1-layer MLP. Thus, we have
\begin{equation}
    \begin{bmatrix}
        \hdots & \vx_1^{(j)} & \vx_2^{(j)} & \hdots & \eq & \vy^{(j)} & \hdots\\
        \hdots & 0 & 0 & \hdots & 0 & 1 & \hdots \\
        \hdots & * & * & \hdots & * & f(\vx_\cdot^{(j)})-\vy^{(j)}  & \hdots\\
        \hdots & * & * & \hdots & * & \|f(\vx_\cdot^{(j)})-\vy^{(j)}\|_1  & \hdots\\
        \hdots & \vZero & \vZero & \hdots & \vZero & \vZero & \hdots\\
        \hdots & 1 & 1 & \hdots & 1 & 1 & \hdots \\
        \hdots & \vp_{jn+1} & \vp_{jn+2} & \hdots & \vp_{jn+n-1} & \vp_{jn+n} & \hdots \\
        \hdots & \vp_{jn+1-s} & \vp_{jn+2-s} & \hdots & \vp_{jn+n-1-s} & \vp_{jn+n-s}& \hdots \\
    \end{bmatrix} 
\end{equation}

Notice that if some task has different dimension than another task, the ``extra'' rows would be zero and will not affect the result. 

For clarity, we set all $*$ values in the $\| \cdot \|_1$ row to $1$s. These will cause the following $\hat{\delta}$ in the following set these to 0. This operation can be omitted as the construction handles these trash values at a later layer.

Let $b$ represent the value of the flag in the second row marking the $\vy$ vectors and let $x$ represent the values in the row with $\|f(\vx^{(j)})-\vy^{(j)}\|_1$. The following ReLUs set the * values to 1.

\begin{equation}
    x \xleftarrow{} x + 1 - \relu(x - Cb) - \relu(Cb - C + 1)
\end{equation}

for some large constant $C$. When $b=0$, this reduces to $x+1-x=1$, and when $b=1$, this reduces to $x+1-1=x$, as desired.

\begin{equation}
    \begin{bmatrix}
        \hdots & \vx_1^{(j)} & \vx_2^{(j)} & \hdots & \eq & \vy^{(j)} & \hdots\\
        \hdots & 0 & 0 & \hdots & 0 & 1 & \hdots \\
        \hdots & * & * & \hdots & * & f(\vx_\cdot^{(j)})-\vy^{(j)}  & \hdots\\
        \hdots & 1 & 1 & \hdots & 1 & \|f(\vx_\cdot^{(j)})-\vy^{(j)}\|_1  & \hdots\\
        \hdots & \vZero & \vZero & \hdots & \vZero & \vZero & \hdots\\
        \hdots & 1 & 1 & \hdots & 1 & 1 & \hdots \\
        \hdots & \vp_{jn+1} & \vp_{jn+2} & \hdots & \vp_{jn+n-1} & \vp_{jn+n} & \hdots \\
        \hdots & \vp_{jn+1-s} & \vp_{jn+2-s} & \hdots & \vp_{jn+n-1-s} & \vp_{jn+n-s}& \hdots \\
    \end{bmatrix} 
\end{equation}

Now define a thresholding function $\hat{\delta}(z)$ that satisfies $\hat{\delta}(0)=1$ and $\hat{\delta}(z)=0$ for $z >> 0$. One such function used here is

\begin{equation}
    \hat{\delta}_C(z) = \relu(1-Cz)
\end{equation}

for some constant $C$, where larger $C$ captures a narrower neighborhood of $0$. 

However, a slight change needs to be added to $\hat{\delta}_C$. In the same row as $\| f(\vx^{(j)}) - \vy^{(j)} \|$ are many values that need to be discarded. Let $b$ be the bit for the current column marking if the column contains an $\vx$ or a $\vy$. We use instead

\begin{equation}
    \hat{\delta}_C(b,z) = \relu(b-Cz)
\end{equation}

This will be zero whenever $b=0$ and $z \geq 0$. We then have as output

\begin{equation}
    \begin{bmatrix}
        \hdots & \vx_1^{(j)} & \vx_2^{(j)} & \hdots & \eq & \vy^{(j)} & \hdots\\
        \hdots & 0 & 0 & \hdots & 0 & 1 & \hdots \\
        \hdots & * & * & \hdots & * & f(\vx_\cdot^{(j)})-\vy^{(j)}  & \hdots\\
        \hdots & 1 & 1 & \hdots & 1 & \|f(\vx_\cdot^{(j)})-\vy^{(j)}\|_1  & \hdots\\
        \hdots & 0 & 0 & \hdots & 0 & \hat{\delta}_C(\|f(\vx_\cdot^{(j)})-\vy^{(j)}\|_1)  & \hdots\\
        \hdots & \vZero & \vZero & \hdots & \vZero & \vZero & \hdots\\
        \hdots & 1 & 1 & \hdots & 1 & 1 & \hdots \\
        \hdots & \vp_{jn+1} & \vp_{jn+2} & \hdots & \vp_{jn+n-1} & \vp_{jn+n} & \hdots \\
        \hdots & \vp_{jn+1-s} & \vp_{jn+2-s} & \hdots & \vp_{jn+n-1-s} & \vp_{jn+n-s}& \hdots \\
    \end{bmatrix} 
\end{equation}

Importantly, $\hat{\delta}_C(\|f(\vx_\cdot^{(j)})-\vy^{(j)}\|_1)=1$ when $f(\cdot)$ is the correct task and $\hat{\delta}_C(\|f(\vx_\cdot^{(j)})-\vy^{(j)}\|_1)=0$ when $f(\cdot)$ disagrees by more than $\frac{1}{C}$ in $L1$-norm.

Lastly, for the next step in the construction, we need to average these soft indicators $\hat{\delta}$ to see how common $f$ is within the context. This is done with an attention layer. Let $\mW_Q$ select the row with all $1$s multiplied by some large constant C, and let $\mW_K$ select the row with flags for results $\vy$. Then

\begin{align}
    \mX^\top \mW_K^\top \mW_Q \mX &= \begin{bmatrix}
        \vdots & \vdots & & \vdots & \vdots \\
        0 & 0 & \hdots & 0 & 0 \\
        0 & 0 & \hdots & 0 & 0 \\
        \vdots & \vdots & & \vdots & \vdots \\
        0 & 0 & \hdots & 0 & 0 \\
        C & C & \hdots & C & C \\
        \vdots & \vdots & & \vdots & \vdots \\
    \end{bmatrix}\\
    \sigma_S(\mX^\top \mW_K^\top \mW_Q \mX) &\approx \begin{bmatrix}
        \vdots & \vdots & & \vdots & \vdots \\
        0 & 0 & \hdots & 0 & 0 \\
        0 & 0 & \hdots & 0 & 0 \\
        \vdots & \vdots & & \vdots & \vdots \\
        0 & 0 & \hdots & 0 & 0 \\
        1/m & 1/m & \hdots & 1/m & 1/m \\
        \vdots & \vdots & & \vdots & \vdots \\
    \end{bmatrix}
\end{align}

where a $1/m$ will appear in every row corresponding to a result $\vy$. Let the value matrix select the row containing $\hat{\delta}(\| f(\vx^{(j)}-\vy^{(j)})\|_1)$. Denote $p=\frac{1}{m}\sum_{j=1}^m \hat{\delta}(\| f(\vx^{(j)}-\vy^{(j)})\|_1)$. Without causal masking, we would have as output

\begin{equation}
    \begin{bmatrix}
        \hdots & \vx_1^{(j)} & \vx_2^{(j)} & \hdots & \eq & \vy^{(j)} & \hdots\\
        \hdots & 0 & 0 & \hdots & 0 & 1 & \hdots \\
        \hdots & * & * & \hdots & * & f(\vx_\cdot^{(j)})-\vy^{(j)}  & \hdots\\
        \hdots & 1 & 1 & \hdots & 1 & \|f(\vx_\cdot^{(j)})-\vy^{(j)}\|_1  & \hdots\\
        \hdots & 0 & 0 & \hdots & 0 & \hat{\delta}_C(\|f(\vx_\cdot^{(j)})-\vy^{(j)}\|_1)  & \hdots\\
        \hdots & p & p & \hdots & p & p  & \hdots\\
        \hdots & \vZero & \vZero & \hdots & \vZero & \vZero & \hdots\\
        \hdots & 1 & 1 & \hdots & 1 & 1 & \hdots \\
        \hdots & \vp_{jn+1} & \vp_{jn+2} & \hdots & \vp_{jn+n-1} & \vp_{jn+n} & \hdots \\
        \hdots & \vp_{jn+1-s} & \vp_{jn+2-s} & \hdots & \vp_{jn+n-1-s} & \vp_{jn+n-s}& \hdots \\
    \end{bmatrix} 
\end{equation}

However, with causal masking, we can only guarantee that $p$ will appear in the columns containing the most recent example being queried. Thankfully, this is all that is needed. 
\end{proof}
\subsection{Task Execution}

\begin{lemma}\label{lem:execution}
    A two layer transformer, with embedding dimension $\cO(d + \log(mn))$ can perform a task and weight its output by the proportion of examples of that task seen within the context.
\end{lemma}

Now that the proportions of each task have been identified in the context, the task itself needs to be executed for the new example being queried. To simplify notation, let the input to this step be 

\begin{equation}
    \mX = \begin{bmatrix}
        \hdots & \vx_1^{(m)} & \vx_2^{(m)} & \hdots & \vx_{n-2}^{(m)} & \eq\\
        \hdots & * & * & \hdots & * & * \\
        \hdots & p & p & \hdots & p & p\\
        \hdots & \vZero & \vZero & \hdots & \vZero & \vZero\\
        \hdots & * & * & \hdots & * & * \\
        \hdots & 1 & 1 & \hdots & 1 & 1 \\
    \end{bmatrix} 
\end{equation}

Following the same process as outlined above, although with slightly different positional encodings, calculate $f(\vx^{(m)})$ and place that result in the final column being decoded. These need to be added at the beginning of the construction, but are only introduced here for clarity.

\begin{equation}
    \begin{bmatrix}
        \hdots & \vx_1^{(m)} & \vx_2^{(m)} & \hdots & \vx_{n-2}^{(m)} & \vx_{n-1}^{(m)}\\
        \hdots & * & * & \hdots & * & * \\
        \hdots & p & p & \hdots & p & p\\
        \hdots & * & * & \hdots & * & f(\vx^{(m)}) \\
        \hdots & \vZero & \vZero & \hdots & \vZero & \vZero\\
        \hdots & * & * & \hdots & * & * \\
        \hdots & 1 & 1 & \hdots & 1 & 1 \\
        \hdots & 0 & 0 & \hdots & 0 & 1 \\
        \hdots & 0 & 0 & \hdots & 1 & 0 \\
    \end{bmatrix} 
\end{equation}

We will transform the row containing all $p$ to be able to approximately multiply $p$ by $f(\vx^{(m)})$. Using the second to last row, perform $p \xrightarrow{} 1-p$. Using the last two rows, clear out the rest of that row and fill it with $-C$ for some large constant $C$. We then have 

\begin{equation}
    \begin{bmatrix}
        \hdots & \vx_1^{(m)} & \vx_2^{(m)} & \hdots & \vx_{n-2}^{(m)} & \vx_{n-1}^{(m)}\\
        \hdots & * & * & \hdots & * & * \\
        \hdots & -C & -C & \hdots & p & 1-p\\
        \hdots & * & * & \hdots & * & f(\vx^{(m)}) \\
        \hdots & \vZero & \vZero & \hdots & \vZero & \vZero\\
        \hdots & * & * & \hdots & * & * \\
        \hdots & 1 & 1 & \hdots & 1 & 1 \\
        \hdots & 0 & 0 & \hdots & 0 & 1 \\
        \hdots & 0 & 0 & \hdots & 1 & 0 \\
    \end{bmatrix} 
\end{equation}

Further, use the second-to-last row to clear out all $*$ in the rows below.

\begin{equation}
    \begin{bmatrix}
        \hdots & \vx_1^{(m)} & \vx_2^{(m)} & \hdots & \vx_{n-2}^{(m)} & \vx_{n-1}^{(m)}\\
        \hdots & * & * & \hdots & * & * \\
        \hdots & -C & -C & \hdots & p & 1-p\\
        \hdots & \vZero & \vZero & \hdots & \vZero & f(\vx^{(m)}) \\
        \hdots & \vZero & \vZero & \hdots & \vZero & \vZero\\
        \hdots & * & * & \hdots & * & * \\
        \hdots & 1 & 1 & \hdots & 1 & 1 \\
        \hdots & 0 & 0 & \hdots & 0 & 1 \\
        \hdots & 0 & 0 & \hdots & 1 & 0 \\
    \end{bmatrix} 
\end{equation}

These previous operations can all be done in a single MLP.

Lastly, use an attention layer where $\mW_K$ selects the row with the $-C$s, $\mW_Q$ selects the row with all $1$s, and $\mW_V$ selects the $f(\vx^{(m)})$. For the last token $\vx_L$,

\begin{align}
    \mX^\top \mW_K^\top \mW_Q \vx_L &= \begin{bmatrix}
        \vdots \\
        -C \\
        -C \\
        \vdots \\
        p \\
        1-p
    \end{bmatrix} [1] = \begin{bmatrix}
        \vdots \\
        -C \\
        -C \\
        \vdots \\
        p \\
        1-p
    \end{bmatrix} \\
    \sigma_S(\mX^\top \mW_K^\top \mW_Q \vx_L) &\approx \begin{bmatrix}
        \vdots \\
        -\infty \\
        -\infty \\
        \vdots \\
        p \\
        1-p
    \end{bmatrix} = \begin{bmatrix}
        \vdots \\
        0 \\
        0 \\
        \vdots \\
        \frac{1}{1+e^{1-2p}} \\
        1-\frac{1}{1+e^{1-2p}}
    \end{bmatrix} 
\end{align}

\begin{align}
    \mW_V \vx_L \sigma_S(\mX^\top \mW_K^\top \mW_Q \vx_L) &= \begin{bmatrix}
        \vdots \\
        \vZero \\ 
        \frac{1}{1+e^{1-2p}}\vZero + (1-\frac{1}{1+e^{1-2p}})f(\vx^{(m)}) \\
        \vZero \\
        \vdots
    \end{bmatrix} \\
    \vx_L + \mW_V \vx_L \sigma_S(\mX^\top \mW_K^\top \mW_Q \vx_L) &= \begin{bmatrix}
        \vx_{n-1}^{(m)} \\
        * \\
        1-p \\ 
        \frac{1}{1+e^{1-2p}}f(\vx^{(m)}) \\
        \vZero \\
        * \\
        1 \\
        1 \\
        0
    \end{bmatrix}
\end{align}

Importantly, we are left with $\frac{1}{1+e^{1-2p}}f(\vx^{(m)})$. The factor $\frac{1}{1+e^{1-2p}}$ is approximately $p$, especially around $\frac{1}{2}$. This multiplication can also be calculated more accurately with approximations using ReLUs or sigmoids, but for brevity and following experimental evidence of a sigmoid shape in task superpositions, these options are ommited.

\subsection{Superposed Tasks with Parallel Heads}
\label{subsec:superposed_tasks}
The above construction works for a single task, where the output is weighted by the proportions of the task within the context. To complete the construction of a transformer that does superposition of tasks, each of these models needs to be placed within the same overall transformer. This is described here.

Let there be a collection of tasks $\{t_i\}_{i=1}^T$ which can be executed by transformers with model weights represented by subscripts ($\cdot_i$). With the input to each transformer being $\mX^{(i)}$, the overall input matrix is given by vertically stacking these matrices.

\begin{equation}
    \mX = \begin{bmatrix}
        \mX_1 \\
        \mX_2 \\
        \vdots \\
        \mX_{T-1} \\
        \mX_{T} 
    \end{bmatrix}
\end{equation}

Similarly, define each MLP's weights and biases as

\begin{equation}
    \mW = \text{diag}(\mW_1, \hdots, \mW_T) \quad \vb = \begin{bmatrix}
        \vb_1 \\
        \vdots \\
        \vb_T
    \end{bmatrix}
\end{equation}

This puts every MLP to be independent of each other. Lastly, we need to change the attention layers. This requires the use of one head per task. In each of the following, $\mW^{(i)}$ is a weight matrix for head $i$, $(\mW)_i$ is the weight matrix for task $i$ in its individual transformer, and each matrix below is in the $i$-th block.

\begin{equation}
    \mW_V^{(i)} = \begin{bmatrix}
        \vdots \\
        \mZero \\
        (\mW_V)_i \\
        \mZero \\
        \vdots
    \end{bmatrix}^\top \quad \mW_K^{(i)} = \begin{bmatrix}
        \vdots \\
        \mZero \\
        (\mW_K)_i \\
        \mZero \\
        \vdots
    \end{bmatrix}^\top \quad \mW_Q^{(i)} = \begin{bmatrix}
        \vdots \\
        \mZero \\
        (\mW_Q)_i \\
        \mZero \\
        \vdots
    \end{bmatrix}^\top
\end{equation}

In all, this model executes multiple tasks in superposition by using parallel streams of heads that each performs a single task. Task identification can happen through the same mechanism as task execution by comparing the output of the task on each in context example with the true output.

For context related tasks, there needs to be positional encodings that allow for looking back a fixed number of tokens. For context agnostic tasks, a wide MLP can be used to approximate arbitrary non-linear transformations of the input. Each of these tasks only require a small number of layers, significantly smaller than those of modern LLMs. It may be possible that LLMs do certain tasks with different combinations of layers. 

Also, if we take the feature $p$ from each parallel stream, this creates the following task identifier.

\begin{equation}
    \vv = \begin{bmatrix}
        p_1 \\
        p_2 \\
        \vdots \\
        p_T
    \end{bmatrix}
\end{equation}

Interpolating between the pure tasks, represented by unit vectors, different amounts of each task will appear in the superposition in roughly equal proportions to those found in $\vv$.

Lastly, we restate this construction formally.

\setcounter{theorem}{0}
\begin{theorem}
    A seven layer transformer with embedding dimension $\cO(d + \log(mn))$ with $K$ heads per attention layer can perform $k$ tasks on vectors of dimension $d$ in superposition, with weighting based on $m$ different in-context examples each of length $n$ .
\end{theorem}

\begin{proof}
    Using in succession each of Lemma \ref{lem:ident}, Lemma \ref{lem:generaltask}, Lemma \ref{lem:context}, and Lemma \ref{lem:execution}, a transformer with the desired properties can execute $k$ tasks in parallel. Lemma \ref{lem:ident} identifies positions within the context that contain the labels $\vy$. Lemma \ref{lem:generaltask} then uses function approximation to perform arbitrary tasks within the architecture, which are then used by \ref{lem:context} to find the proportions of each task and aggregate them into a single task identifier. Lastly, Lemma \ref{lem:execution} uses this task identifier to create a weighted sum of outputs from the different tasks based on their in-context proportions. 
\end{proof}

\begin{remark}
    Transformers of greater depth than seven layers can also represent this construction by setting the weights in all other layers for the non residual part to zero.
\end{remark}

\end{document}